%% file: main.tex
\documentclass{article}

\usepackage[round]{natbib}

\input{math_commands.tex}

\usepackage{hyperref}
\usepackage{url}
\usepackage{amsthm}
\usepackage{subcaption}
\usepackage{graphicx}
\usepackage{xcolor}
\usepackage[margin=1in]{geometry}
\usepackage{algorithm}
\usepackage{booktabs}
\usepackage[noend]{algpseudocode}
\usepackage{dsfont}

\newtheorem{corollary}{Corollary}

\newtheorem{theorem}{Theorem}
\newtheorem{lemma}{Lemma}

\newtheorem{definition}{Definition}

\newtheorem{example}{Example}

\numberwithin{equation}{section} %
\numberwithin{lemma}{section} %
\numberwithin{theorem}{section} %
\numberwithin{definition}{section} %
\numberwithin{corollary}{section} %

\usepackage{amsfonts}
\def \R {\mathbb{R}}

\begin{document}

\title{On the Neural Feature Ansatz for Deep Neural Networks}
\author{Edward Tansley\thanks{Mathematical Institute, Woodstock Road, University of Oxford, Oxford, UK, OX2 6GG. {\tt edward.tansley@maths.ox.ac.uk.} This author's work was supported by the Mathematics of Random Systems CDT.},\quad Estelle Massart\thanks{ICTEAM Institute, UCLouvain, Euler Building, Avenue Georges Lemaˆıtre, 4 - bte L4.05.01, Louvain-la-Neuve, B - 1348, Belgium. {\tt estelle.massart@uclouvain.be}} \quad and \quad Coralia Cartis\thanks{Mathematical Institute, Woodstock Road, University of Oxford, Oxford, UK, OX2 6GG. {\tt coralia.cartis@maths.ox.ac.uk.} This author's work was supported by the Hong Kong Innovation and Technology Commission
(InnoHK Project CIMDA) and by the EPSRC grant EP/Y028872/1, Mathematical Foundations of Intelligence: An “Erlangen Programme” for AI.}}

\date{October 6, 2025}

\maketitle

\begin{abstract}
  Understanding feature learning is an important open question in establishing a mathematical foundation for deep neural networks.
  The Neural Feature Ansatz (NFA) states that after training, the Gram matrix of the first-layer weights of a deep neural network is proportional to some power $\alpha>0$ of the average gradient outer product (AGOP) of this network with respect to its inputs. Assuming gradient flow dynamics with balanced weight  initialization, the NFA was proven to hold throughout training for two-layer linear networks with exponent $\alpha = 1/2$ (Radhakrishnan et al., 2024).
  We extend this result to networks with $L \geq 2$ layers, showing that the NFA holds with exponent $\alpha = 1/L$, thus demonstrating a depth dependency of the NFA.
  Furthermore, we prove that for unbalanced initialization, the NFA holds asymptotically through training if weight decay is applied.
  We also provide counterexamples showing that the NFA does not hold for some network architectures with nonlinear activations, even when these networks fit arbitrarily well the training data. 
  We thoroughly validate our theoretical results through numerical experiments across a variety of optimization algorithms, weight decay rates and initialization schemes.
\end{abstract}

\section{Introduction}
\label{sec:intro}

Deep neural networks (DNN) typically operate in the overparametrized regime, where the number of parameters to tune in the model outweighs the size of the training data. While overparametrization endows DNNs with extreme expressivity, allowing exact interpolation of the data even when the latter are noisy realizations \citep{zhang_understanding_2021}, the good performance of DNNs observed in practice calls for \emph{implicit biases} that prevent the model to overfit the data \citep{vardi_implicit_2023}. 

This work addresses recent developments, exploring specific biases in the feature learning mechanisms, namely, the process through which neural networks extract information from high-dimensional input data. Recently, the Neural Feature Ansatz (NFA)  was proposed as a possible explanation of feature extraction, in which model weights reflect the importance that features have on model predictions (quantified by the magnitude of the partial derivatives of the model output with respect to its input). The NFA was shown empirically to hold in a wide range of models including feedforward, convolutional, recurrent neural networks as well as transformers \citep{radhakrishnan_mechanism_2024}, and was used as a posit to shed light on other behaviors such as neural collapse \citep{beaglehole_average_2024}, training instabilities \citep{zhu_catapults_2024}, and grokking \citep{mallinar_emergence_2025}. Developing a theoretical foundation for the NFA is however still an open question.

Beyond the NFA, the presence of low-dimensional structures was identified in \citet{parkinson_relu_2023, parkinson_relu_2025} in the specific setting of deep linear neural networks (with a single final ReLU layer) trained with weight decay; this type of architecture is known to provide insight into the effect of network depth \citep{arora_optimization_2018}, despite the fact that adding linear networks does not increase model expressivity. \citet{parkinson_relu_2025} showed both theoretically and empirically that increasing model depth leads to some form of  implicit bias on the linear layers, that induces the learned model to exhibit some low-dimensional behavior, by varying mostly along a subset of directions of the input space and being almost constant along its orthogonal complement.  Functions that are only varying along a low-dimensional subspace of their input space are often referred to as \emph{multi-index}, or \emph{low-rank} functions. This implicit bias also improves model generalization, which, in the case where data is generated using a multi-index target function, is higher when the low-dimensional subspace of variation is aligned with the subspace of variation of the target function. In a similar vein,  \citet{guth_rainbow_2024} explored feature extraction mechanisms across layers in random feedforward neural networks, and identified some low-rank structure and dimensionality reduction mechanisms within layers. Furthermore, the ubiquity of multi-index models motivated the development of dedicated training strategies, see \citet{bruna_survey_2025} for a survey. In particular, it was shown in \citet{mousavi-hosseini_robust_2025} that learning the low-rank structure can remove the dependence on the ambient dimension in high-dimensional settings. 

While these two lines of work both aim to advance the understanding of feature learning mechanisms,  many open questions remain. In this work, inspired by 
\citet{parkinson_relu_2023, parkinson_relu_2025} and
\citet{radhakrishnan_mechanism_2024}, we aim to uncover the role of depth in the NFA. More precisely, we prove that 
for deep linear neural networks trained under gradient flow dynamics with balanced initialization, the NFA holds exactly for all time. We further prove that this result holds asymptotically in the case of unbalanced initialization, in the presence of weight decay regularization. For nonlinear neural networks, we propose a counterexample that illustrates that the NFA does not always hold, and a second one illustrating the limitations of the NFA to account for model generalization.  We conclude the paper with numerical experiments supporting our theoretical findings as well as the dimensionality reduction mechanism resulting from the NFA when learning multi-index functions. 

Meanwhile,  the  recent work \citet{ziyin_formation_2025} derived a unifying framework involving variants of the NFA under assumptions involving alignment of gradients, features and/or weights. Even more recently, and concurrently to our work, the authors of \citet{boix-adsera_features_2025} proposed an alternative to the NFA that can be derived from first-order optimality conditions. Our results here add further understanding to this growing body of works.

\paragraph{Notation:} We consider a neural network $f_{\vtheta}$, parametrized by a set of parameters $\vtheta$, aiming to approximate some function $f^{*}$ using a set of $N$ data points $\{(\vx_{i}, y_{i})\}_{1\leq i \leq N}$, with $\vx_{i} \in \R^{d}$ and $y_{i} = f^{*}(\vx_i)$. This paper addresses $L$-layers feedforward neural networks, whose set of parameters $\vtheta$ contains weight matrices $\mW_1, \dots, \mW_L$ and biases $\vb_1, \dots, \vb_L$. The network $f_{\vtheta}$ is learned by minimizing the empirical loss $\mathcal{L}(\vtheta) = \frac{1}{N}\sum_{i=1}^{N} l_{\vtheta}(\vx_{i}, y_{i})$, for some  $l_{\vtheta} : \R \times \R \to \R$. We note $\mW_{l,t}$ the weight matrix associated with the $l$th layer at iteration $t$ of the optimization algorithm, and $\mW_{l,0}$ its corresponding value at initialization. For a given function $f$, we denote by $\Af$ the AGOP of $f$, which we define in \autoref{sec:preliminaries}. We use $\|\cdot\|_{2}$ and $\|\cdot\|_{F}$ to refer to the 2-norm and the Frobenius norm, respectively. We use $\Tr(\cdot)$ for the trace and $\ker(\cdot)$ for the kernel of a matrix. We use bold lower and uppercase characters (e.g.,  $\vv,\ \mW$) for vectors and matrices, respectively. 

\section{Preliminaries}
\label{sec:preliminaries}

\paragraph{Low-rank functions.}
Low-rank functions, also referred to as \emph{multi-index} \citep{bruna_survey_2025},  \emph{multi-ridge} \citep{tyagi_learning_2014}, functions with \emph{active subspaces} \citep{constantine_active_2015}, or functions of \textit{low effective dimensionality} \citep{cartis_learning_2024}, are functions that vary only within a (low-dimensional and unknown) linear subspace and are constant along its orthogonal complement.  These functions satisfy the following equivalent properties.

\begin{definition}%
[\citet{cartis_learning_2024}]\label{lem:low:rank:characterisation}
A continuously differentiable\footnote{Note that some of the neural networks we consider are not continuously differentiable, for example due to the classical ReLU activation, but as these models are continuously differentiable almost everywhere this does not raise any practical issue.}
function $f:\R^{d} \rightarrow \R$ has rank $r \leq d$ if it satisfies one of the following equivalent conditions:
\begin{enumerate}
\item There exists a subspace $\mathcal{T}$ of dimension $r$ such that $f(\vx_{\top}+\vx_{\perp}) = f(\vx_{\top})$ for all $\vx_{\top}\in \mathcal{T}$ and $\vx_{\perp} \in \mathcal{T}^{\perp}$.
\item There exists a subspace $\mathcal{T}$ of dimension $r$ such that $\nabla f(\vx)\in \mathcal{T}$ for all $\vx\in \R^{d}$.%
\item There exists a matrix $\mA\in \R^{r \times d}$ and a map $g:\R^r \rightarrow \R$ such that $f(\vx) = g(\mA\vx)$ for all $\vx\in \R^{d}$. 
\end{enumerate}
\end{definition}

\paragraph{Average Gradient Outer Product (AGOP).}
We recall the definition of the \textit{average gradient outer product (AGOP)} of the network with respect to the data, which we write $\mA_f \in \R^{d \times d}$:
\begin{equation}
    \mA_{f} := \frac{1}{N} \sum_{i=1}^{N} \nabla f(\vx_i) \nabla f(\vx_i)^{\top}.
\end{equation}
Note that the eigenvectors associated with the eigenvalues of $\mA_{f}$ with largest magnitude correspond to the directions in which perturbations to the data have the largest effect on the network output,  when averaged over the data points. Indeed, let $\vz$ be some arbitrary vector, then 
\begin{equation} \label{eq:partial_deriv}
    \frac{1}{N} \sum_{i = 1}^N \|\nabla f(\vx_i)^{\top}\vz \|_{2}^{2} = \vz^{\top}\mA_{f}\vz; 
\end{equation}
choosing $\vz$ as the eigenvector of $\mA_f$ associated with its largest eigenvalue will maximize the right-hand side of \eqref{eq:partial_deriv} over all unit-norm vectors (by the definition of the Rayleigh quotient), hence, the left hand-side of \eqref{eq:partial_deriv}.  
Note that, if the network has a multivariate output, a similar expression involving the Jacobian can be used \citep{radhakrishnan_mechanism_2024}. Noting that the rows of the Jacobian are themselves the transposes of the gradients, for $f(\vx) = (f_1(\vx),\ldots,f_m(\vx))^{\top}$, we have that $\mJ_f(\vx)^{\top}\mJ_{f}(\vx) = \sum_{j=1}^{m} \nabla f_{j}(\vx) \nabla f_{j}(\vx)^{\top}$.

\paragraph{The Neural Feature Ansatz (NFA).}  The  \textit{Neural Feature Ansatz} (see \citet{radhakrishnan_mechanism_2024}) states that the weight matrix associated to the first layer can explain the structure of the AGOP matrix:
\begin{equation}
     \mW_{1}^{\top}\mW_{1}\propto (\mA_{f})^{\alpha}
\end{equation}
for some $\alpha > 0$. A value of $\alpha = 1/2$ is proposed and proven in the case of a 2-layer linear network (under gradient flow and balanced initialization) \citep{radhakrishnan_mechanism_2024}.

In \autoref{fig:NFA_visualized}, we include an illustration of the NFA. We see that after training a 5-layers network with weight decay, $\mW_{1}^{\top}\mW_{1}\propto (\mA_{f})^{1/5}$. Furthermore, these two matrices are approximately equal, supporting our results in \autoref{sec:nfa_proof}.

\begin{figure*}[t]
    \centering
    \begin{subfigure}{0.22\textwidth}
        \centering
        \includegraphics[width=\textwidth]{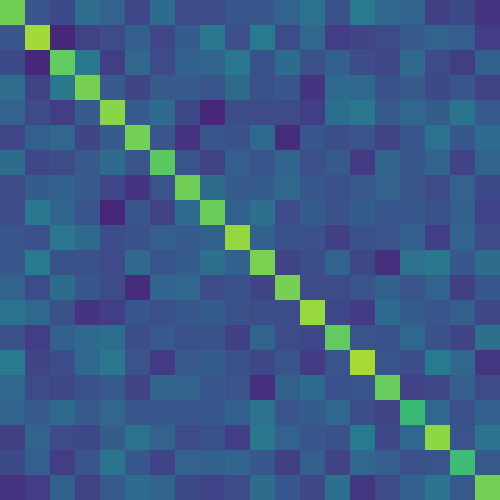}
        \caption{}
    \end{subfigure}
    \hfill
    \begin{subfigure}{0.22\textwidth}
        \centering
        \includegraphics[width=\textwidth]{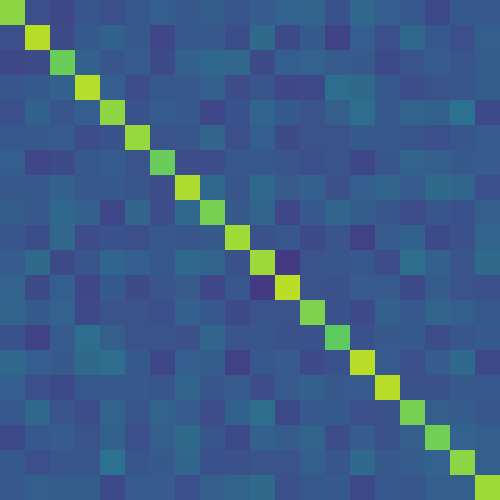}
        \caption{}
    \end{subfigure}
    \hfill
    \begin{subfigure}{0.22\textwidth}
    \includegraphics[width=\textwidth]{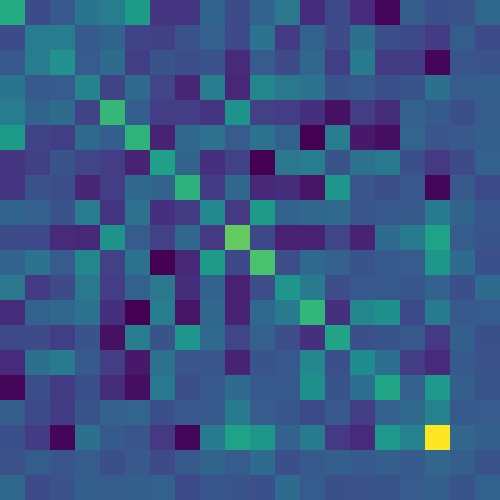}
    \caption{}
    \end{subfigure}
    \hfill
    \begin{subfigure}{0.2655\textwidth}
        \centering
        \includegraphics[width=\textwidth]{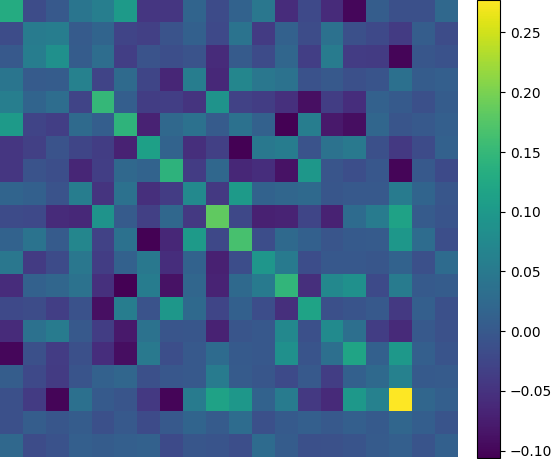}
        \caption{}
    \end{subfigure}
    \caption{\textbf{The NFA holds after training even with unbalanced initialization.}\\(a) $\WonetWone$ before training. (b) $(\Af)^{1/L}$ before training. (c) $\WonetWone$ after training. (d) $(\Af)^{1/L}$ after training, with $L=5$ linear layers. Alignment as measured by cosine similarity (\autoref{def:cosine_sim}): before training = 0.915; after training = 1.000. We plot the same experiment as in \autoref{fig:cosine_similarity_though_time} (a) which is described in \autoref{sec:numerics}.}
    \label{fig:NFA_visualized}
\end{figure*}

\section{The NFA for deep linear networks}
\label{sec:nfa_proof}

In this section, we prove that the NFA holds (at least asymptotically) for deep linear neural networks when the latter are trained with gradient flow dynamics and weight decay regularization. Thus we extend  the results of  \citet{radhakrishnan_mechanism_2024}, that were restricted to 2-layer NNs. Throughout this section we shall consider deep linear networks of the form
\begin{equation}
    f(\vx) = \mW_{L}\mW_{L-1}\cdot\ldots\cdot\mW_{1}\vx
\end{equation}
for some $L \geq 2$.
The Jacobian of such a network is given by $\mJ_{f}(\vx) = \mW_{L}\mW_{L-1}\cdot\ldots\cdot\mW_{1}$ for all $\vx \in \R^{d}$. We therefore write $\mJ_f$ instead of $\mJ_f(\vx)$. Due to this constant Jacobian, there holds $\Af = \mJ_{f}^{\top}\mJ_{f}$.

Here we assume the gradient flow dynamics, namely, 
\begin{equation}
    \frac{\partial\mW_{l, t}}{\partial t} = -\frac{\partial \mathcal{L}(\vtheta_t)}{\partial \mW_{l, t}},\label{eq:gradient_flow}
\end{equation}

where $\mW_{l, t}$ and $\vtheta_t$ are respectively the weights of the $l$th layer and the total set of parameters at time $t \geq 0$.

As a first step, we assume balanced initialization of the weight matrices in the model, where balancedness is defined as follows. 
\begin{definition}[Balanced matrices]
    Two weight matrices $\mW_{l} \in \R^{k \times m}$ and $\mW_{l+1} \in \R^{n \times k}$ are said to be balanced if $\mW_{l}\mW_{l}^{\top}$ = $\mW_{l+1}^{\top}\mW_{l+1}$.
\end{definition}
We say that a network is balanced if the weight matrices of each pair of adjacent layers are balanced and an initialization of the network weights is balanced if the network is balanced at time $t=0$.

The following result states that if this balancedness property holds at initialization, it holds for all $t \geq 0$ assuming the weights follow gradient flow dynamics.

\begin{lemma}[\citet{arora_optimization_2018}]\label{lem:balancedness_preserved}
    For time $t \geq 0$, let us define $f_{t}(\vx) = \mW_{L,t}\mW_{L-1,t}\cdots \mW_{1,t}\vx$ for $\vx \in \R^{d}$.
    Suppose that $\mW_{1,t},\mW_{2,t},\ldots,\mW_{L,t}$ follow the gradient flow dynamics given by \autoref{eq:gradient_flow}, then for any $t\geq0$ and $1 \leq l < L$, there holds
    \begin{equation}
        \mW_{l+1,t}^{\top}\mW_{l+1,t} - \mW_{l,t}\mW_{l,t}^{\top} =  \mW_{l+1,0}^{\top}\mW_{l+1,0} - \mW_{l,0}\mW_{l,0}^{\top}
    \end{equation}
\end{lemma}
Hence if the layers in a network are balanced at $t = 0$, they are balanced for all $t \geq 0$ when the network is trained under gradient flow. 

We now present a result proving that the NFA holds for $L$-layer linear networks, extending the result of \citet{radhakrishnan_mechanism_2024}, where 2-layer linear networks are considered.
This result suggests that the $\alpha$ value in the Neural Feature Ansatz has a depth dependency, rather than being a fixed value such as $1/2$.

\begin{theorem}[NFA for deep linear networks]\label{thm:1_over_l}
    For $t \geq 0$, let $f_{t}(\vx) = \mW_{L,t}\mW_{L-1,t}\cdots \mW_{1,t}\vx$ for $\vx \in \R^{d}$.
    Suppose that $\mW_{1,t},\mW_{2,t},\ldots,\mW_{L,t}$ follow the gradient flow dynamics given by \autoref{eq:gradient_flow}. Suppose additionally that $\mW_{1,0},\mW_{2,0},\ldots,\mW_{L,0}$ are initialized to be balanced ($\mW_{l,0}\mW_{l,0}^{\top} = \mW_{l+1,0}^{\top}\mW_{l+1,0}$ for $l=1,\ldots,L-1$) then at any time $t\geq0$, there holds
    \begin{equation}
        \mW_{1,t}^{\top}\mW_{1,t} = \left ( \mA_{f,t} \right )^{1/L}.
    \end{equation}
    where $\mA_{f, t} = \mJ_{f_t}(\vx)^{\top}\mJ_{f_t}(\vx)$.
\end{theorem}

\begin{proof}
    By first applying \autoref{lem:balancedness_preserved} and expanding the brackets, we have, for all $k \geq 1$, $t\geq 0$ and $1 \leq l < n$:
    \begin{align}
                \mW_{l,t}^{\top}(\mW_{l+1,t}^{\top}\mW_{l+1,t})^{k}\mW_{l,t} &= \mW_{l,t}^{\top}(\mW_{l,t}\mW_{l,t}^{\top})^{k}\mW_{l,t} \nonumber \\
                &= (\mW_{l,t}^{\top}\mW_{l,t})^{k+1}.
    \label{eq:1_over_l_helper}
    \end{align}
    As $f_t$ is linear, there holds $\mJ_{f_t}(\vx) = \mW_{L,t}\mW_{L-1,t}\cdot \ldots \cdot \mW_{1,t}$. We may repeatedly apply the equality of \autoref{eq:1_over_l_helper} to see that\footnote{For brevity, we include the explicit steps in the Appendix.}
        \begin{align}
         \left ( \mJ_{f_t}(\vx)^{\top}\mJ_{f_t}(\vx) \right )^{1/L} &=
          \left ( (\mW_{1,t}^{\top}\mW_{1,t})^{L} \right )^{1/L} \nonumber \\
         &=\mW_{1,t}^{\top}\mW_{1,t}. \label{eq:NFA_balanced_result}
        \end{align}
\end{proof}

Therefore, for deeper linear networks, the NFA holds with $\alpha = 1/L$, where $L$ is the number of layers in the model.

\textbf{Remark:} We assumed that the network output was multivariate. The assumption that the layers are all balanced means that in the case of a univariate output, all of the layers would have to be rank 1 both at initialization and throughout training. Indeed, it has been shown that the weight matrices in deep linear networks converge to be rank 1 for classification tasks \citep{ji_gradient_2018}.

\subsection{Removing the balancedness assumption}

For unbalanced initializations, \autoref{lem:balancedness_preserved} states that the weights in adjacent layers shall remain unbalanced through training. By introducing weight decay into the gradient flow dynamics, it can be shown that the weights in adjacent layers will become asymptotically balanced \citep{kobayashi_weight_2024}. We recall the gradient flow dynamics with weight decay:

\begin{equation}
    \frac{\partial\mW_{l, t}}{\partial t} = -\frac{\partial \mathcal{L}(\vtheta_t)}{\partial \mW_{l, t}} - \lambda\mW_{l, t},\label{eq:gradient_flow_wd}
\end{equation}
where $\lambda > 0$ is the weight decay parameter.

We also recall the following lemma which we will use to prove a similar result to \autoref{thm:1_over_l} for unbalanced initialization.

\begin{lemma}[\citet{kobayashi_weight_2024}]\label{lem:exp_balanced}
    Suppose $\mW_{l},\ \mW_{l+1}$ are the weight matrices of two adjacent layers of a neural network, that has a loss function differentiable with respect to $\hat{\mW}_{l+1, l} :=\mW_{l+1}\cdot\mW_{l}$. Suppose that the layers follow the gradient flow dynamics given by \autoref{eq:gradient_flow_wd} for $\lambda > 0$, then $\mW_{l+1, t}^{\top}\mW_{l+1, t} - \mW_{l, t}\mW_{l, t}^{\top}$ converges exponentially quickly to zero. In particular, $\mW_{l, t}\mW_{l, t}^{\top} - \mW_{l+1, t}^{\top}\mW_{l+1, t} = e^{-2\lambda t}\mC_{l}$ where $\mC_l = \mW_{l, 0}\mW_{l, 0}^{\top} - \mW_{l+1, 0}^{\top}\mW_{l+1, 0}$.
\end{lemma}

Letting $c_{\max} := \max_{l} \|\mC_{l}\|_{F}$ in  \autoref{lem:exp_balanced}, we can prove the following theorem, the proof of which is included in the Appendix.

\begin{theorem}[NFA for deep linear networks, without balanced initialization]\label{thm:1_over_l_unbalanced}
    For time $t \geq 0$, let $f_{t}(\vx) = \mW_{L,t}\mW_{L-1,t}\cdots \mW_{1,t}\vx$ for $\vx \in \R^{d}$.
    Suppose that $\mW_{1,t},\mW_{2,t},\ldots,\mW_{L,t}$ follow the gradient flow dynamics given by \autoref{eq:gradient_flow_wd} for $\lambda > 0$. Defining $c_{\max}$ as above, at any time $t>0$, there holds
    \begin{equation}
         \| \mA_{f,t} - \left ( \mW_{1,t}^{\top}\mW_{1,t} \right )^{L}\|_{F} = \mathcal{O}(c_{\max}e^{-2\lambda t}).
    \end{equation}
    where $\mA_{f, t} = \mJ_{f_t}^{\top}\mJ_{f_t}$ and $\lambda > 0$ is the weight decay constant.
\end{theorem}

According to \citet{wihler_holder_2009}, for two $d \times d$ positive semi-definite matrices $\mX, \mY$, 
\begin{equation}
\|\mX^{1/L} - \mY^{1/L}\|_{F} \leq d^{(L-1)/2L}\|\mX-\mY\|_{F}^{1/L}.\label{eq:Wihler_bound}
\end{equation}
This result allows us to quantify directly the gap between the neural feature matrix $\mW_1^\top \mW_1$ and the $L$-th principal square root of the AGOP matrix.

\begin{corollary}\label{cor:1_over_n_unbalanced_form}
For time $t \geq 0$, let $f_{t}(\vx) = \mW_{L,t}\mW_{L-1,t}\cdots \mW_{1,t}\vx$ for $\vx \in \R^{d}$.
    Suppose that $\mW_{1,t},\mW_{2,t},\ldots,\mW_{L,t}$ follow the gradient flow dynamics given by \autoref{eq:gradient_flow_wd} for $\lambda > 0$. Defining $c_{\max}$ as above, at any time $t>0$, there holds
    \begin{equation}
         \| (\mA_{f,t})^{1/L} - \mW_{1,t}^{\top}\mW_{1,t} \|_{F} = \mathcal{O}(c_{\max}e^{-2\lambda t/L}).
    \end{equation}
    where $\mA_{f, t} = \mJ_{f_t}^{\top}\mJ_{f_t}$ and $\lambda > 0$ is the weight decay constant. 
\end{corollary}

\textbf{Remark:} Note that \autoref{lem:exp_balanced} holds for the linear part of networks of the form
\begin{equation}
    f(\vx) = \va^{\top}\phi(\mW_{L}\mW_{L-1}\cdot\ldots\cdot\mW_{1}\vx + \vb_1) + b_2
\end{equation}
where $\phi$ is a differentiable activation (which can be seen as a classification head on top of the linear feature extraction layers). We consider this type of architecture in our numerical experiments (although with a ReLU activation function which that is not differentiable at the origin). Specifically, the results is applicable for the evolution of $\tilde{f}(\vx) = \mW_{L} \cdot \ldots \cdot \mW_{1}\vx + \vb_1$.

\section{The NFA for nonlinear networks}
\label{sec:NFA_nonlinear}

While previous section showed that the NFA holds for deep linear neural networks under suitable assumption on the training process, we show now that there exist functions and architectures such that the NFA does not hold, even when the network function $f$ exactly matches the true function $f^{*}$.%

\begin{example}
Suppose that $f^{*}:\R^{2} \rightarrow \R$ is defined by $f^{*}(\vx) = [x_1]_{+} + [x_2]_{+}$ and that we have some data set $\{(\vx_{i}, y_{i})\}_{1\leq i \leq N}$, with the $\vx_i$s drawn from some distribution $\mX$ that has equal probability for each of the four quadrants (e.g. $U([-1,1]^{2})$).  We observe that $f$ can be expressed exactly by a one-hidden-layer bias-free neural network with ReLU activation in the hidden layer: 
$$f^*(x) = \va^\top \phi(\mW \vx),$$
with 
\begin{equation*}
    \mW = \begin{pmatrix}
        1 & 0 \\
        0 & 1
    \end{pmatrix}; \quad \va = \begin{pmatrix}
        1 \\
        1
    \end{pmatrix}; \quad \mW^\top\mW = \begin{pmatrix}
        1 &0\\
        0 &1
    \end{pmatrix}.
\end{equation*}
Moreover, since this function has gradient discontinuities on the lines $x_1 = 0$ and $x_2 = 0$, the preactivations in the hidden layer will also have to align with these directions so that this is the only one-hidden-layer bias-free representation of this function up to rescaling of the rows of $\mW$ and corresponding entries of $\va$. For any $\vx \in \R^2$, there holds
\begin{align*}
    \nabla f(\vx) &= \begin{pmatrix}
        \mathds{1}_{\{x_1 > 1\}} \\
        \mathds{1}_{\{x_2 > 1\}}
    \end{pmatrix};\\
    \nabla f(\vx)\nabla f(\vx)^{\top} &= \begin{pmatrix}
        \mathds{1}_{\{x_1 > 1\}} & \mathds{1}_{\{x_1 > 1;\ x_2 > 1\}}\\
        \mathds{1}_{\{x_1 > 1;\ x_2 > 1\}} & \mathds{1}_{\{x_2 > 1\}}
    \end{pmatrix},
\end{align*}
where we use $\mathds{1}$ to denote an indicator function. By our distributional assumption, we have
\begin{equation*}
    \Ef := \E_{\vx \sim \mX} \left [ \nabla f(\vx)\nabla f(\vx)^{\top} \right ] = \frac{1}{4}\begin{pmatrix}
        2 & 1\\
        1 & 2
    \end{pmatrix}.
\end{equation*}
Assuming $N$ is large, we will have $\Af \approx \Ef$. In fact, by the strong law of large numbers, we have $\Af \rightarrow \Ef$ as $N \rightarrow \infty$. On the other hand, there is no power of $\alpha > 0$ such that $\WtW \propto (\Ef)^{\alpha}$, and so the NFA does not hold in this setting.

\end{example}

From this counterexample, we may deduce NFA does always hold for nonlinear networks (regardless of the value of $\alpha > 0$).

\paragraph{What about wider or deeper networks?} In the above example, we considered a narrow two-layer network which was not overparameterized. The question remains of what happens for wider or deeper networks. In the case of overparameterized two-layer neural networks, we can add zero-weight connections to find networks for which $f^{*}$ is interpolated and the NFA holds. For example, suppose that we have
\begin{equation*}
    \mW = \begin{pmatrix}
        1 & 0 \\
        0 & 1 \\
        1  & 1
    \end{pmatrix}; \quad \va = \begin{pmatrix}
        1 \\
        1 \\
        0
    \end{pmatrix}; \quad \mW^\top\mW = \begin{pmatrix}
        2 &1\\
        1 &2
    \end{pmatrix},
\end{equation*}

the NFA would hold exactly with $\alpha = 1$. Of course, the third neuron in the hidden layer would have no impact on the network output. By the universal approximation property it can show that any function $f^{*}:\R^{2} \rightarrow \R$ can be approximated by a sufficiently wide 2-layer network (with bias). 

\paragraph{The NFA and generalization. } We next show that alignment between $\mA_f$ and $\mA_{f^{*}}$ is neither necessary nor sufficient for $f$ to be a good fit of $f^{*}$.

Regarding the lack of sufficiency: setting $f(\vx) := f^{*}(\vx) + c$ for some constant $c$ gives  $\mA_f = \mA_{f^{*}}$, while $f$ is a very poor fit of $f^{*}$ for large values of $c$.

Regarding the lack of necessity: for $n \geq 1$, define $f_{n}^{*}:\R^{2} \rightarrow \R$ by
\begin{equation}
    f_{n}^{*}(\vx) =  \frac{1}{n}\cos(n^{2}x_1) + x_2,
\end{equation}
for $\vx = (x_1,\ x_2)$. Note that
\begin{align*}
    \nabla f_{n}^{*}(\vx) &= 
    \begin{pmatrix}
        -n\sin(n^{2}x_1)\\
        1
    \end{pmatrix} \\ 
    \nabla f_{n}^{*}(\vx) \nabla f_{n}^{*}(\vx)^{\top} &= 
    \begin{pmatrix}
        n^{2}\sin^{2}(n^{2}x_1) & -n\sin(n^{2}x_1)\\
        -n\sin(n^{2}x_1) &1
    \end{pmatrix}.
\end{align*}

Assuming that we sample data points such that the distribution of $x_1$ is symmetric around the origin, for example, $x_1 \sim U[-\pi,\ \pi]$, the off-diagonal terms are zero. In this instance, we have that
\begin{equation}
    \mA_{f_n^*} = 
    \begin{pmatrix}
        \mathcal{O}(n^{2}) & 0\\
        0 &1
    \end{pmatrix}.
\end{equation}
Suppose we set $f(\vx):= x_2$, then
\begin{equation}
    \mA_{f} = 
    \begin{pmatrix}
        0 & 0\\
        0 &1
    \end{pmatrix},
\end{equation}

and hence $\cos(\mA_f, \mA_{f^{*}_{n}}) \rightarrow 0$ as $n \rightarrow \infty$. Notice that $\E[|f(\vx) - f_{n}^{*}(\vx)|] < 1/n \rightarrow 0$ as $n \rightarrow \infty$. There is therefore no obvious implication between the NFA and model generalization.

\section{Numerical Experiments}
\label{sec:numerics}

\begin{figure*}
    \centering
    \begin{subfigure}{0.48\textwidth}
        \centering
        \includegraphics[width=\textwidth, height=0.5\textwidth]{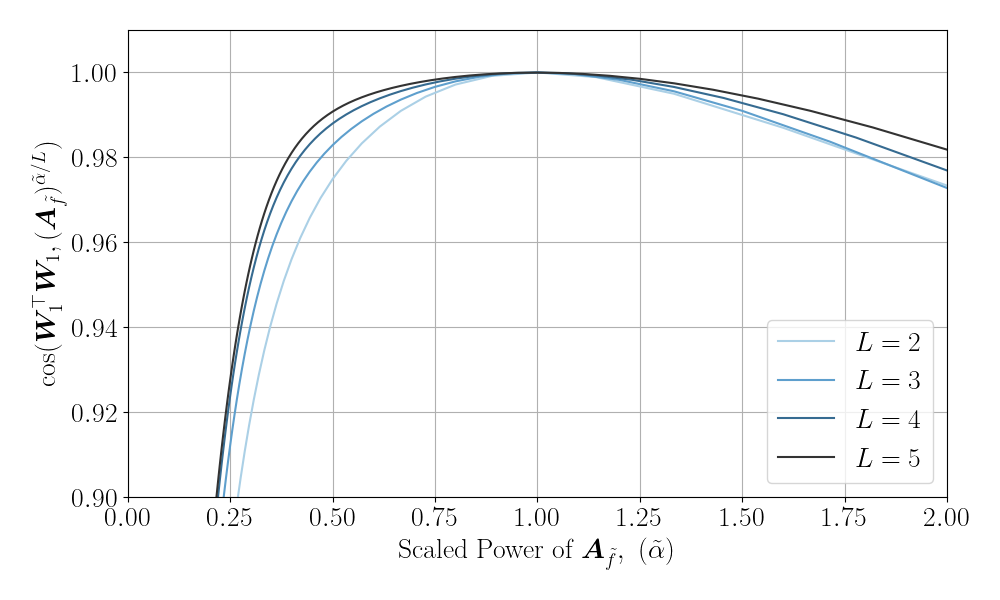}
        \caption{unbalanced initialization, $\lambda = 10^{-2}$}
    \end{subfigure}
    \hfill
    \begin{subfigure}{0.48\textwidth}
        \centering
        \includegraphics[width=\textwidth, height=0.5\textwidth]{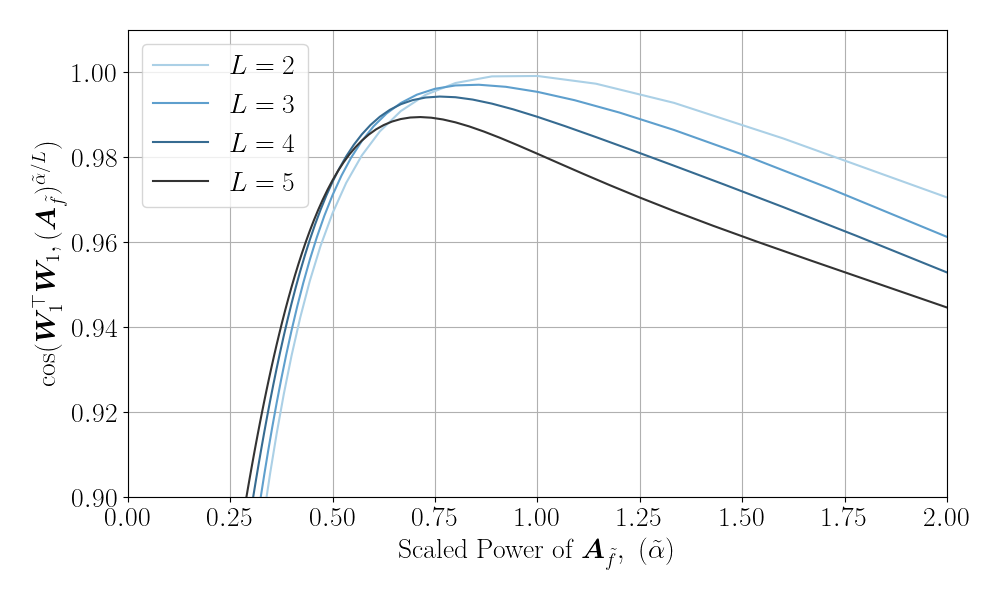}
        \caption{unbalanced initialization, $\lambda = 10^{-3}$}
    \end{subfigure}
    \begin{subfigure}{0.48\textwidth}
        \centering
        \includegraphics[width=\textwidth, height=0.5\textwidth]{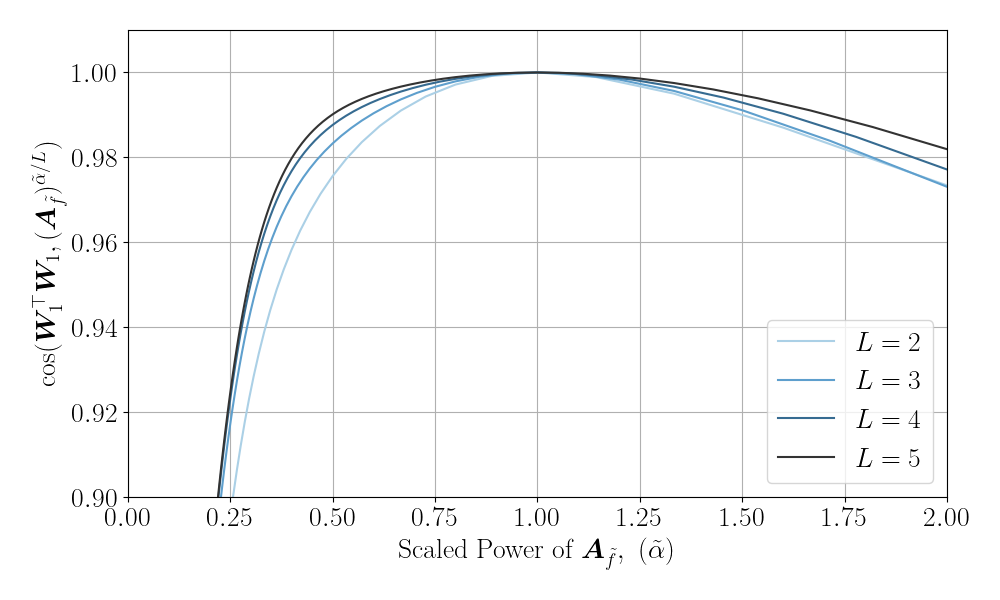}
        \caption{balanced initialization, $\lambda = 10^{-2}$}
    \end{subfigure}
    \hfill
    \begin{subfigure}{0.48\textwidth}
        \centering
        \includegraphics[width=\textwidth, height=0.5\textwidth]{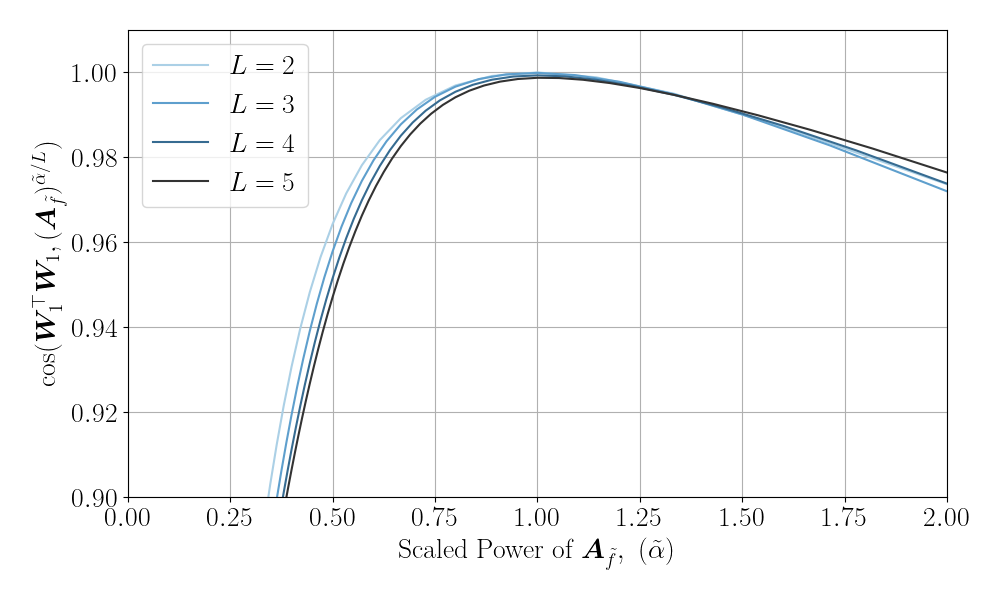}
        \caption{balanced initialization, $\lambda = 10^{-3}$}
    \end{subfigure}
    \caption{Illustration of the impact of initialization and weight decay ($\lambda$) on the alignment between $\WonetWone$ and $(\Af)^{1/L}$ at the end of training (SGD). A learning rate of $\eta = 10^{-4}$ was used.}
    \label{fig:scaled_alpha_plots}
\end{figure*}

We now conduct numerical experiments to verify the claims in \autoref{sec:nfa_proof}, as well as exploring further the low-dimensional structure resulting from the NFA when approximating low-rank functions.  Across this section, we rely on the data generation mechanisms proposed in \citet{parkinson_relu_2025}. 

\paragraph{Data generation.}
We consider low-rank target functions of the form
\begin{equation*}
    f^{*}(\vx) = \va^{\top}g(\mA\vx + \vb_1),\label{eq:experiment_form}
\end{equation*}
for some matrix $\mA \in \R^{r \times d}$, where $g$ is some link function, and for input dimension $d = 20$. In this section, we let $g : x \mapsto [x]_{+}$ be the ReLU function applied elementwise and $r = 5$; results associated with $r \in \{2,\ 20\}$, and other link functions are presented in the appendices. The datapoints are generated according to $\vx_i \sim U([-1/2, 1/2]^{20})$ for all $i$, and $y_i = f^*(\vx_i)$ (see also the appendices for additional results with label noise). Unless stated otherwise, we use a dataset size of 2048 points. For the stochastic optimization methods, we used a batch size of 64.

\subsection{Validating our theoretical results}

Note that the theoretical results derived in \autoref{sec:nfa_proof} were obtained under the assumption that model training follows a gradient flow dynamics. In practice, this algorithm is discretized and possibly replaced by stochastic optimizations methods such as SGD with or without momentum, or Adam. This section aims to assess whether the findings of \autoref{sec:nfa_proof} still (at least approximately) hold in these settings.

\paragraph{Architectures considered.} Following \citet{parkinson_relu_2025}, we considered here deep linear neural networks with a single ReLU final layer\footnote{Our notation differs as we have $L$ linear layers prior to the nonlinearity rather than $L-1$}: 
\begin{equation}
    f(\vx) = \va^{\top}[\mW_{L}\mW_{L-1}\cdot\ldots\cdot\mW_{1}\vx + \vb_1]_{+} + b_2.\label{eq:architecture_for_experiments}
\end{equation}
This network structure allows approximating more general functions than linear ones, but results for deep linear neural networks are provided in the Appendix.
The number of layers is variable, but each hidden layer has width 64. To align with the theory of \autoref{sec:nfa_proof}, we consider the initial linear part of this network evolves through training, namely, $\tilde{f}(\vx) = \mW_{L}\mW_{L-1}\cdot\ldots\cdot\mW_{1}\vx + \vb_1$. The Jacobian of this function is given by $\mJ_{\tilde{f}} = \mW_{L}\mW_{L-1}\cdot\ldots\cdot\mW_{1}$ and we let $\AfTilde := \mJ_{\tilde{f}}^{\top}\mJ_{\tilde{f}}$.
To assess the validity of the NFA, and in accordance with \citet{radhakrishnan_mechanism_2024}, we calculate the cosine similarity, whose definition is recalled next, between $(\AfTilde)^\alpha$ and $\mW_{1}^{\top}\mW_{1}$ for various powers of $\alpha$.

\begin{definition}\label{def:cosine_sim}
    The cosine similarity between two matrices $\mM$ and $\mN$ is given by
    $\cos(\mM, \mN):= \Tr(\mM^{\top}\mN)/(\|\mM\|_{F} \cdot\|\mN\|_{F})$.
\end{definition}

\paragraph{Initialization schemes and training algorithms considered.} We consider both balanced and unbalanced initialization schemes. For the unbalanced initialization scheme, we use the default PyTorch initialization for linear layers\footnote{Each weight $w$ in the $l$th layer is initialized as $w \sim U(-1/\sqrt{d_l},\ 1/\sqrt{d_l})$ where $d_l$ is the in-degree of the $l$th layer}, while the balanced initialization scheme is described in Appendix.

We consider a variety of optimization algorithms for model training. In this section, we primarily include results for GD and for stochastic gradient descent (SGD), with and without momentum. Additional results for training networks with Adam \citep{kingma_adam_2014}, and more results for gradient descent, are included in Appendix. Following \citet{parkinson_relu_2025}, we train each model for $60,000$ epochs before reducing the learning rate by a factor of 10 and running an additional $100$ epochs. We provide a description of the algorithm hyperparameters in Appendix.

\paragraph{Results. } 
\autoref{fig:scaled_alpha_plots} displays the cosine similarity between the neural feature matrix $\WonetWone$ and the AGOP of the linear part of the model as described above, with respect to the NFA exponent $\alpha$. In order to display on the same plot curves obtained from networks with different numbers of layers, we rescale $\alpha$ by the number of linear layers. In other words, the $x$-axis of \autoref{fig:scaled_alpha_plots} is $\tilde{\alpha} := L\alpha$. According to \autoref{thm:1_over_l}, the cosine similarity should be the greatest when $\tilde{\alpha} = 1$, which corresponds to $\alpha = 1/L$, which is indeed the case for balanced initialization (regardless of the weight decay parameter value and despite the fact that the training algorithm is SGD instead of a mere gradient flow dynamics in \autoref{sec:nfa_proof}). Note also that, as long as the weight decay parameter $\lambda$ is sufficiently large, the cosine similarity is also maximal when $\tilde{\alpha} = 1$ for unbalanced initialization.

\begin{figure*}
    \centering
    \begin{subfigure}{0.48\textwidth}
        \centering
        \includegraphics[width=\textwidth, height=0.5\textwidth]{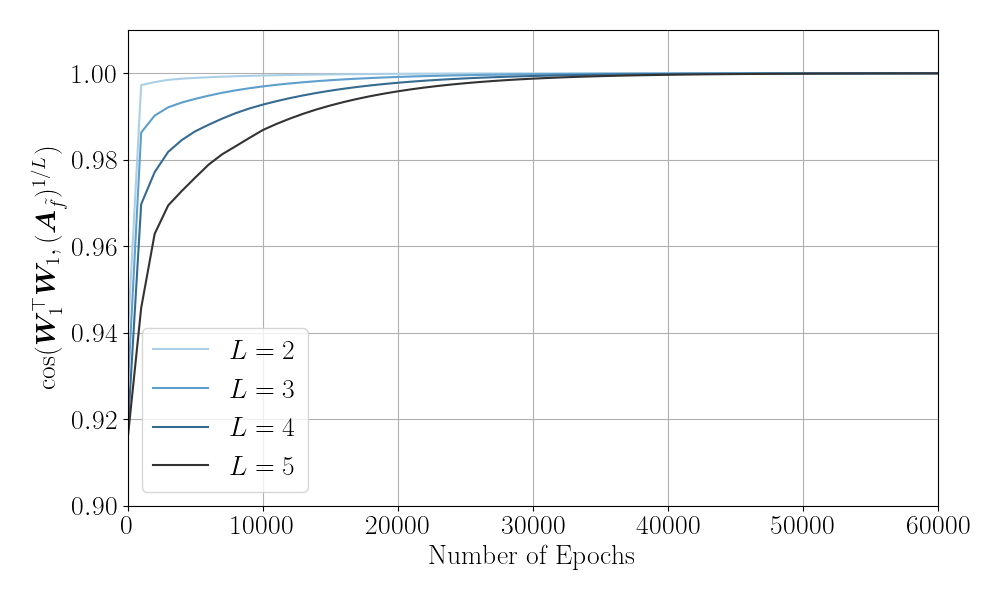}
        \caption{SGD, no momentum}
    \end{subfigure}
    \hfill
    \begin{subfigure}{0.48\textwidth}
        \centering
        \includegraphics[width=\textwidth, height=0.5\textwidth]{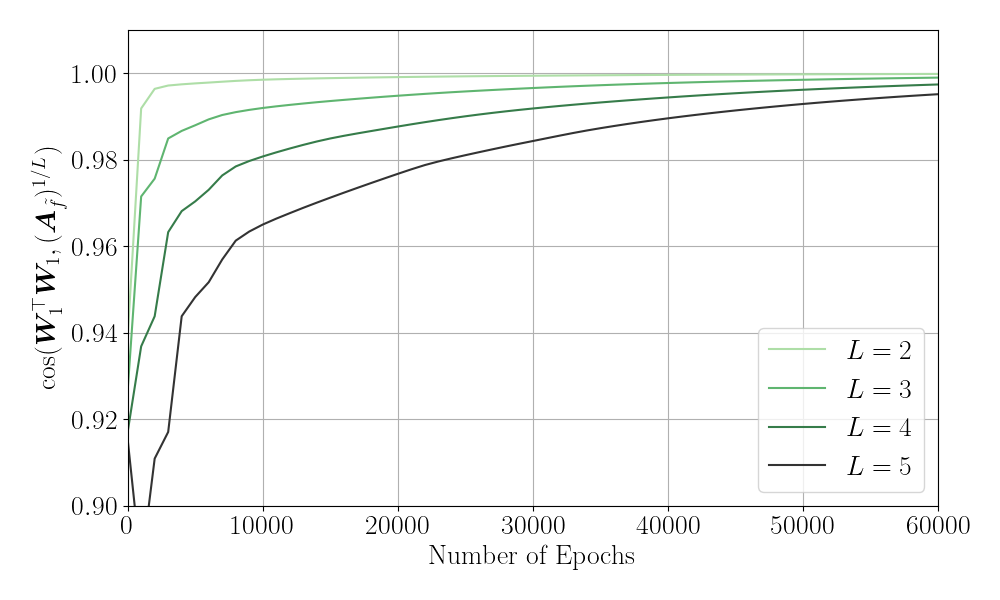}
        \caption{GD, no momentum}
    \end{subfigure}
    \begin{subfigure}{0.48\textwidth}
    \includegraphics[width=\textwidth, height=0.5\textwidth]{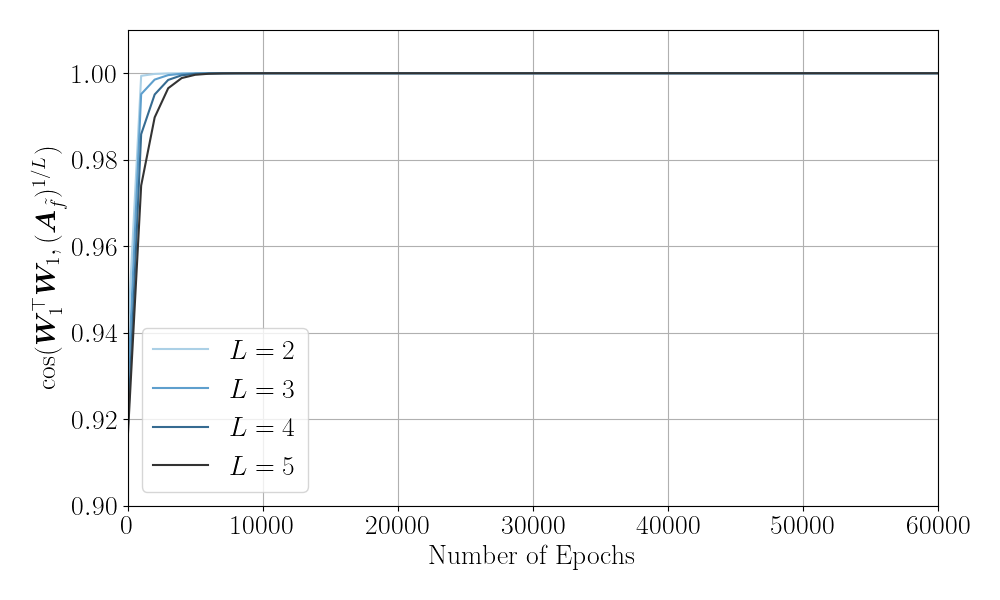}
        \caption{SGD, momentum}
    \end{subfigure}
    \hfill
    \begin{subfigure}{0.48\textwidth}
        \centering
        \includegraphics[width=\textwidth, height=0.5\textwidth]{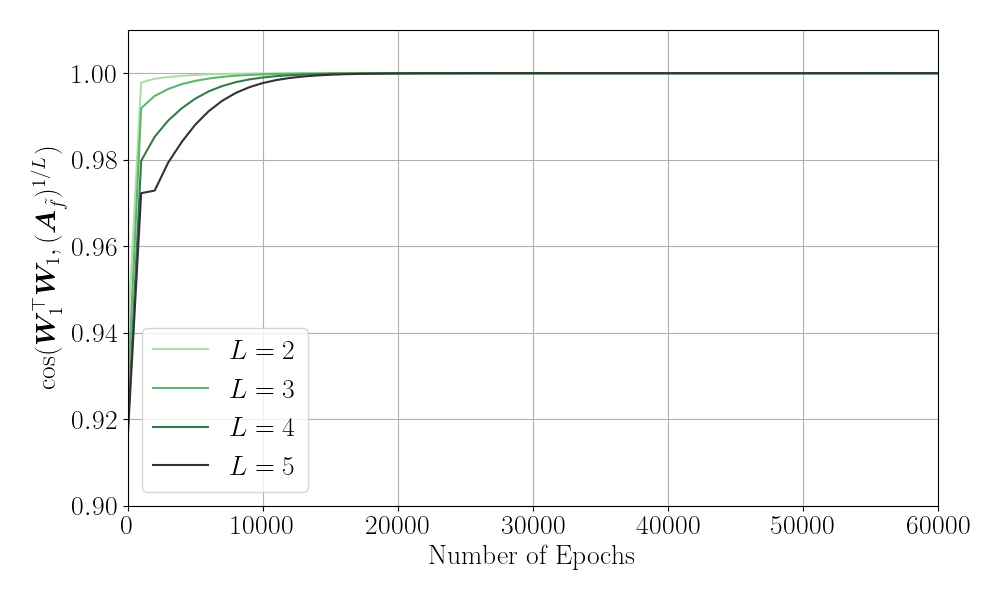}
        \caption{GD, momentum}
    \end{subfigure}
    \caption{Illustration of the impact of the optimization algorithm on the NFA, with weight decay ($\lambda = 10^{-2}$). When momentum is used, it is weighted by a parameter $\beta = 0.9$. The learning rates for SGD and GD are $\eta = 10^{-4}$ and $\eta = 10^{-3}$, respectively. Here we plot the first $60,000$ epochs before the learning rate decrease.}
    \label{fig:cosine_similarity_though_time}
\end{figure*}

\autoref{fig:cosine_similarity_though_time} compares the impact of the choice of the optimizer on the NFA, showing that the NFA holds in both settings, and that furthermore momentum increases the rate at which $\WonetWone$ and $(\AfTilde)^{1/L}$ align for both SGD and GD.

Finally, regarding depth dependency, note that all our experiments show that the rate at which $\WonetWone$ and $(\AfTilde)^{1/L}$ align is slower for deeper networks than for shallower networks.

\begin{figure*}
    \centering
    \begin{subfigure}{0.48\textwidth}
        \centering
        \includegraphics[width=\linewidth]{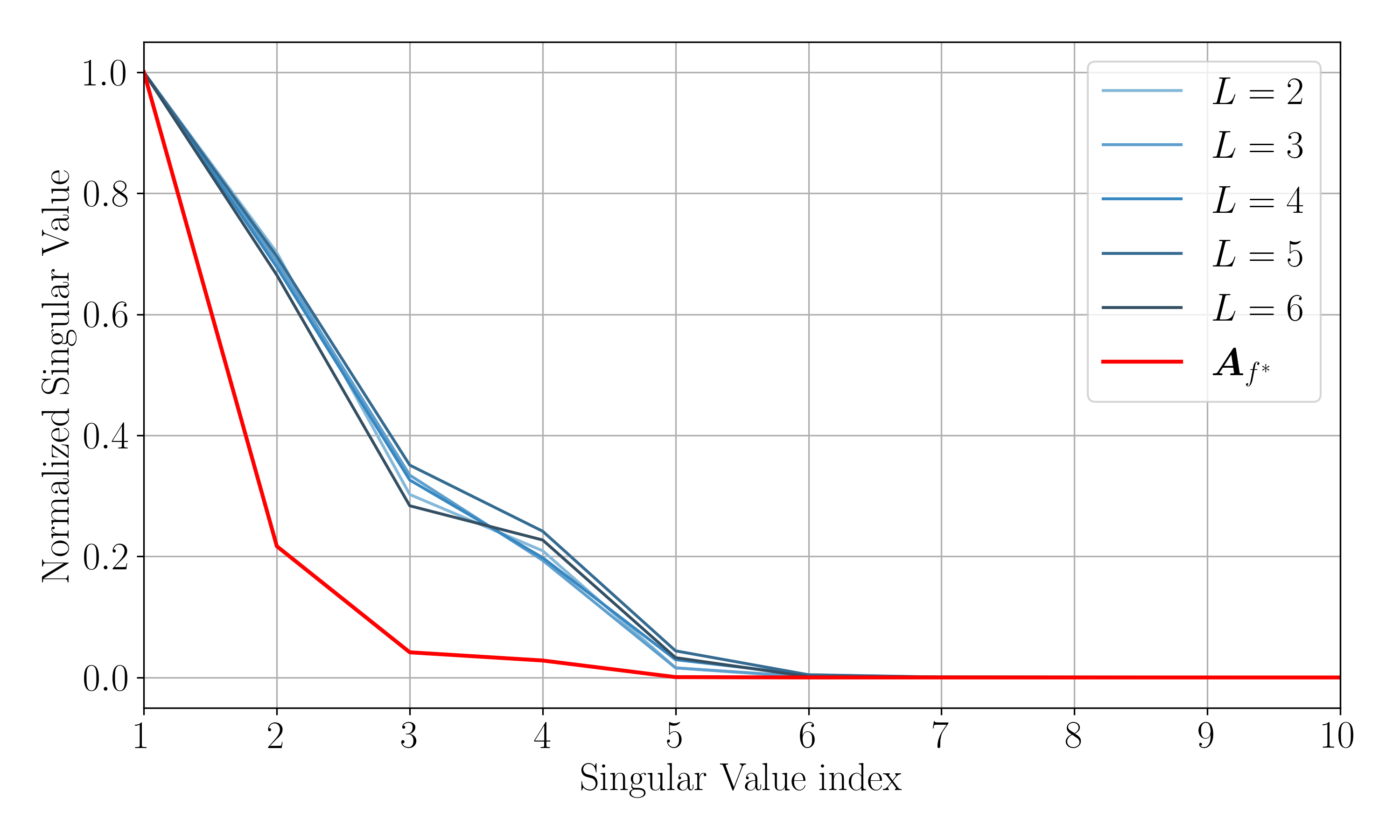}
        \caption{8192 Data points}
    \end{subfigure}
    \hfill
    \begin{subfigure}{0.48\textwidth}
        \centering
        \includegraphics[width=\linewidth]{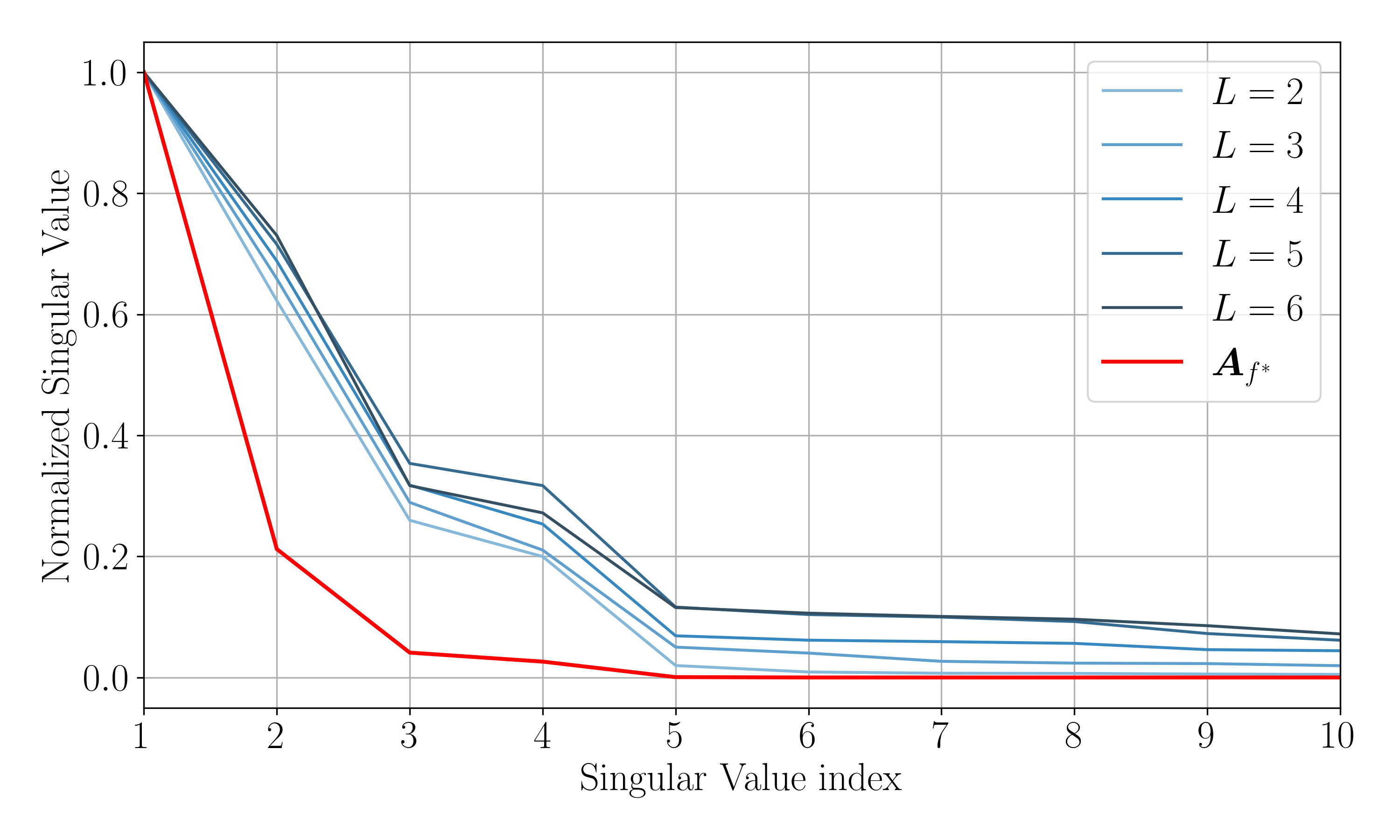}
        \caption{2048 Data points}
    \end{subfigure}
    \caption{Singular values decay for $\WonetWone$ vs $\AfStar$. Normalized Singular Values are computed by dividing by $\sigma_1$ for the respective matrix}
    \label{fig:nonlinear_network}
\end{figure*}

\subsection{Recovering the low-rank structure of the target function}

Note that, since our target function $f^{*}$ is low-rank, by \autoref{lem:low:rank:characterisation} $\nabla f^{*}(\vx) \in \mathcal{T}$ for all $\vx \in \R^{d}$, where $\mathcal{T}$ is the subspace of variation of $f^*$.
Therefore, $\nabla f^{*}(\vx)^{\top}\vx_{\perp} = 0$ for all $\vx \in \R^{d},\ \vx_{\perp} \in \mathcal{T}^{\perp}$ and hence  $(\vx_{\perp})^{\top}\mA_{f^{*}}\vx_{\perp} = 0$ for all $\vx_{\perp} \in \mathcal{T}^{\perp}$.
As such, the low-rank structure of $f^{*}$ is captured by $\mA_{f^{*}}$ as $\mathcal{\mathcal{T}}^{\perp} \subseteq \ker(\mA_{f^{*}})$\footnote{We use a ``subset'' notation here as the sample data points do not cover the entire domain}.
Proportionality between $\WonetWone$ and some power of $\Af$ implies that the two matrices must have the same rank. Hence, supposing that the NFA holds and $f$ has the same low-rank structure as the target function $f^*$, the first-layer weight matrix $\mW_1$ must have low-rank structure, which we verify here numerically in this section for a generic feedforward neural network architecture, with depth ranging from 2 to 5, width 64 and same initialization as above. We however have biases in each layer, as well as ReLU activations. 
The data generation mechanism is the same as in the previous section, but we also experiment with $N= 8192$ alongside $N= 2048$.
The model was trained with SGD with weight decay ($\lambda = 10^{-3}$) and momentum ($\beta = 0.9$).
\autoref{fig:nonlinear_network} shows the distribution of the singular values of $\mW_1$, for ReLU networks of various depths.
Note that the low-rank structure of the target function $f^{*}$ is indeed captured by the first-layer weight matrix for all  network depths considered, and that this becomes increasingly accurate as the number of data points increases, allowing a better learning of the low-rank structure in the target function.

\section{Conclusion}

We have shown that under gradient flow and balanced initialization, the NFA holds for deep linear networks with a depth-dependent exponent.
Furthermore, assuming that weight decay is applied, the NFA holds asymptotically regardless of initialization.
As a barrier to extending the NFA for linear networks to more general feedforward neural networks,  in \autoref{sec:NFA_nonlinear}, we show that there exist functions that can be expressed by a given architecture for which perfect proportionality is not attainable between some power of the AGOP and some power of $\WonetWone$ for that architecture.
In \autoref{sec:numerics}, we illustrate that  our theoretical results for gradient flow continue to hold when applying diverse training algorithms.
We also illustrate that in the case of nonlinear networks, $\WonetWone$ has the same low-rank structure as the AGOP of the target function, which indicates that the low-rank behaviour observed for linear networks may be extendable in the future to nonlinear ones.

% \clearpage

\bibliography{references}

\newpage
\appendix

\section{Omitted Proofs}

\subsection{Derivation of \autoref{eq:NFA_balanced_result}}

Using the definition of $\mJ_{f_t}$, we obtain the first equality below,
\begin{align*}
         \left ( \mJ_{f_t}(\vx)^{\top}\mJ_{f_t}(\vx) \right )^{1/L} &=
         \left ( \mW_{1,t}^{\top}\cdots \mW_{L-1,t}^{\top} \mW_{L,t}^{\top}\mW_{L,t}\mW_{L-1,t}\cdots \mW_{1,t} \right )^{1/L} \\
         &=
         \left ( \mW_{1,t}^{\top}\cdots \mW_{L-2,t}^{\top}(\mW_{L-1,t}^{\top}\mW_{L-1,t})^{2}\mW_{L-2,t}\cdots \mW_{1,t} \right )^{1/L} \\
         &=
         \left ( \mW_{1,t}^{\top}\cdots \mW_{L-3,t}^{\top}(\mW_{L-2,t}^{\top}\mW_{L-2,t})^{3}\mW_{L-3,t}\cdots \mW_{1,t} \right )^{1/L} \\
         &\vdots\\
         &=\left ( \mW_{1,t}^{\top}(\mW_{2,t}^{\top}\mW_{2,t})^{L-1}\mW_{1,t} \right )^{1/L} \\
         &= \left ( (\mW_{1,t}^{\top}\mW_{1,t})^{L} \right )^{1/L}\\
         &= \mW_{1,t}^{\top}\mW_{1,t},
\end{align*}
where to obtain the remaining equalities, we repeatedly
apply \eqref{eq:1_over_l_helper}.

\subsection{Proof of \autoref{thm:1_over_l_unbalanced}}

We prove the following two lemmas to simplify the proof of this theorem. 

\begin{lemma}\label{lem:norm_bound}
    For $t \geq 0$, let $f_{t}(\vx) = \mW_{L,t}\mW_{L-1,t}\cdots \mW_{1,t}\vx$ for $\vx \in \R^{d}$. Suppose that $\mW_{1,t},\mW_{2,t},\ldots,\mW_{L,t}$ follow the gradient flow dynamics given by \autoref{eq:gradient_flow_wd} for $\lambda > 0$.
    Assume that a continuously differentiable loss function $\mathcal{L}$ is bounded below, i.e., there exists $\mathcal{L}_{low}$ such that $\mathcal{L}(\vtheta) \geq \mathcal{L}_{low}$ for all $\vtheta$.
    Then there exists some constant $C_F$ such that $\|\mW_{l, t}\|_{F} \leq C_F$ for all $l\in \{1,\ldots,L\}$ and for all $t \geq 0$.
\end{lemma}
\begin{proof}
    Note that \eqref{eq:gradient_flow_wd} is the gradient flow of the regularized loss function
    \begin{equation}
        \hat{\mathcal{L}}_{\lambda}(\vtheta) :=  \mathcal{L}(\vtheta) + \frac{\lambda}{2}\sum_{l=1}^{L}\|\mW_{l}\|_{F}^{2}. 
    \end{equation}
    Therefore, as $\mathcal{L}$ is continuously differentiable, $\hat{\mathcal{L}_{\lambda}(\vtheta_{t})}$ is monotonically decreasing with respect to $t$. For all $t \geq 0$, we have that
    \begin{equation}
        \mathcal{L}(\vtheta_{t}) + \frac{\lambda}{2}\sum_{l=1}^{L}\|\mW_{l, t}\|_{F}^{2} \leq \mathcal{L}(\vtheta_{0}) + \frac{\lambda}{2}\sum_{l=1}^{L}\|\mW_{l, 0}\|_{F}^{2},
    \end{equation}
    so that, for any $l\in \{1,\ldots,L\},\ t \geq 0$,
    \begin{align}
        \|\mW_{l, t}\|_{F}^2 &\leq \frac{2}{\lambda} \left ( \mathcal{L} (\vtheta_{0}) - \mathcal{L} (\vtheta_{t}) + \frac{\lambda}{2}\sum_{j=1}^{L}\|\mW_{j, 0}\|_{F}^{2} - \frac{\lambda}{2}\sum_{\substack{j=1 \\ j\neq l}}^{L}\|\mW_{j, t}\|_{F}^{2} \right )\\
        &\leq \frac{2}{\lambda} \left ( \mathcal{L}(\vtheta_{0}) - \mathcal{L}_{low} + \frac{\lambda}{2}\sum_{j=1}^{L}\|\mW_{j, 0}\|_{F}^{2} \right ) =: C_{F}^{2}.
    \end{align}
    Hence, $\|\mW_{l, t}\|_{F} \leq C_F$, for all $l\in \{1,\ldots,L\}$ and for all $t \geq 0$.
\end{proof}

The following lemma may be thought of as a generalization of \eqref{eq:1_over_l_helper} to the case of unbalanced layers.

\begin{lemma}
 For any $l \in \{1, \dots, L\}$, let $\mD_{l, t} := (\prod_{j=1}^{l-1}\mW_{j, t}^{\top}) (\mW_{l, t}^{\top}\mW_{l, t})^{L-l+1}(\prod_{j=1}^{l-1}\mW_{j, t}^\top)^\top$, then $\|\mD_{l+1, t} - \mD_{l, t}\|_{F} \leq2^{(L-l)}e^{-2\lambda t}c_{\max}C_F^{2(L-l)}$, where  $C_F$ is defined above and $c_{\max} := \max_{l} \|\mC_{l}\|_{F}$, for $\mC_l := \mW_{l, 0}\mW_{l, 0}^{\top} - \mW_{l+1, 0}^{\top}\mW_{l+1, 0}$.
\end{lemma}

\begin{proof}
    From the definitions of $\mD_{l, t},\ \mD_{l+1,t}$, there holds
    \begin{equation}
    \mD_{l+1, t} - \mD_{l, t} = (\prod_{j=1}^{l-1}\mW_{j, t}^{\top})
    [\mW_{l, t}^{\top}(\mW_{l+1,t}^{\top}\mW_{l+1,t})^{L-l}\mW_{l, t} -  (\mW_{l, t}^{\top}\mW_{l, t})^{L-l+1}] (\prod_{j=1}^{l-1}\mW_{j, t}^{\top})^\top.
    \end{equation}

    Then, by the submultiplicity of the Frobenius norm and by \autoref{lem:norm_bound}, 
    \begin{align}
        \|\mD_{l+1, t} -  \mD_{l}\|_{F} &\leq (\prod_{j=1}^{l-1} \|\mW_{j, t}^{\top}\|_F) \cdot \|\mW_{l, t}^{\top}(\mW_{l+1,t}^{\top}\mW_{l+1,t})^{L-l}\mW_{l, t} -  (\mW_{l, t}^{\top}\mW_{l, t})^{L-l+1}\|_{F} \cdot (\prod_{j=1}^{l-1}\|\mW_{j, t}^{\top}\|_F)\\
        &\leq C_F^{2(l-1)}\|\mW_{l, t}^{\top}(\mW_{l+1,t}^{\top}\mW_{l+1,t})^{L-l}\mW_{l, t} -  (\mW_{l, t}^{\top}\mW_{l, t})^{L-l+1}\|_{F}.
    \end{align}
    From \autoref{lem:exp_balanced}, we have that $\mW_{l+1,t}^{\top}\mW_{l+1,t} = \mW_{l, t}\mW_{l, t}^{\top} - e^{-2\lambda t}\mC_{l}$. It follows
    \begin{align}
        \| \mW_{l, t}^{\top}(\mW_{l+1,t}^{\top}\mW_{l+1,t})^{L-l}\mW_{l, t} -  (\mW_{l, t}^{\top}\mW_{l, t})^{L-l+1}\|_F &= \|\mW_{l, t}^{\top}(\mW_{l, t}\mW_{l, t}^{\top}- e^{-2\lambda t}\mC_{l})^{L-l}\mW_{l, t} -  (\mW_{l, t}^{\top}\mW_{l, t})^{L-l+1}\|_F  \label{eq:interm_estelle} \\
        &= \| \sum_{j=1}^{2^{(L-l)} - 1}\mW_{l, t}^{\top} \mT_{j, t} \mW_{l, t}\|_F \\
        &\leq C_{F}^{2}\sum_{j=1}^{2^{(L-l)} - 1} \| \mT_{j, t} \|_F,
    \end{align}
    where the $\mT_{j, t}$s correspond to the terms in the binomial expansion of $(\mW_{l, t}\mW_{l, t}^{\top}+ e^{-2\lambda t}\mC_{l})^{L-l}$ that have at least one power of $e^{-2\lambda t}\mC_{l}$, which is all terms other than $(\mW_{l, t}\mW_{l, t}^{\top})^{L-l}$ (indeed, the term $(\mW_{l, t}\mW_{l, t}^{\top})^{L-l}$ will cancel with the second term of the right-hand side of \eqref{eq:interm_estelle}).
    There are $2^{(L-l)} -1$ such terms.
    For each of these terms, assuming that $e^{-2\lambda t}c_{\max} \leq C_F$, which holds asymptotically in time, there holds $\|\mT_{j, t}\|_F \leq e^{-2\lambda t}c_{\max}C_F^{2(L-l-1)}$. It follows that for $t$ sufficiently large
    \begin{equation}
        \|\mD_{l+1, t} -  \mD_{l}\|_{F} \leq C_{F}^{2}(2^{(L-l)}-1)e^{-2\lambda t}c_{\max}C_{F}^{2(L-l-1)} \leq 2^{(L-l)}e^{-2\lambda t}c_{\max}C_{F}^{2(L-l)}.
    \end{equation}
    
\end{proof}

Using these two above Lemmas, we are now able to prove \autoref{thm:1_over_l_unbalanced}.

\begin{proof}[Proof of \autoref{thm:1_over_l_unbalanced}]

Observing that $\Af_{t} = \mD_{L, t}$ and $( \mW_{1,t}^{\top}\mW_{1,t}  )^{L} =\mD_{1, t}$,
we may form a telescoping sum and use the triangle inequality, asymptotically through time to show that we have
\begin{align}
    \| \mA_{f,t} - \left ( \mW_{1,t}^{\top}\mW_{1,t} \right )^{L}\|_{F} &= \|\sum_{l = 1}^{L-1} (\mD_{l+1, t} - \mD_{l, t}) \|_{F}\\
    &\leq  \sum_{l = 1}^{L-1}\| \mD_{l+1, t} - \mD_{l, t} \|_{F} \\
    &\leq \sum_{l = 1}^{L-1}2^{(L-l)}e^{-2\lambda t}c_{\max}\hat{C}_F^{2(L-l)} \\
    &\leq e^{-2\lambda t}c_{\max}\hat{C}_F^{2L}\sum_{l = 1}^{L-1}2^{(L-l)}\\
    &= 2^{L}e^{-2\lambda t}c_{\max}\hat{C}_F^{2L},\label{eq:no_balancedness_convergence}
\end{align}
where $\hat{C}_F = \max(C_{F},\ 1)$.
As $\lambda > 0$ by assumption, this expression decays exponentially to zero as $t \rightarrow \infty$.
\end{proof}

\section{Experiment setup}

In this section, we include further details on the settings for our numerical experiments.

\paragraph{Forcing balanced initialization}
The proof of \autoref{thm:1_over_l} relies upon a balanced initialization of the weight matrices. We describe here the scheme we used to ensure that the initialization of the network is balanced, which is used in the experiments to produce \autoref{fig:scaled_alpha_plots}. Let $d_l$ denote the in-degree of the $l$th layer. In our experiments on synthetic data, there holds $d_1=20$, but $d_l = 64$ for $l = 2, \ldots, L$. We state the following Lemma, which for completeness, we prove in \autoref{sec:lem:W1_rescaling_proof}\footnote{We consider the square of the Frobenius norm to simplify the algebraic expressions.}.
\begin{lemma}\label{lem:W1_rescaling}
Suppose the entries of the weight matrix $\mW_{l} \in \R^{d_{l+1} \times d_{l}}$ are drawn independently according to the uniform distribution $U(-1/\sqrt{d_l}, 1/\sqrt{d_l})$, then 
\begin{equation}
    \mathbb{E}[\|\mW_{l}^{\top}\mW_{l}\|_{F}^{2}] = \frac{d_{l+1}}{d_{l}} \left ( \frac{1}{5} + \frac{d_{l}+d_{l+1}-2}{9}  \right ), 
\end{equation}
\end{lemma}

As a result, $\mathbb{E} \left [ \|\mW_{1}\mW_{1}^{\top}\|_{F}^{2} \right ] \neq \mathbb{E} \left [ \|\mW_{2}^{\top}\mW_{2}\|_{F}^{2} \right ]$. Thus we use the following scaling of the weights of $\mW_{1}$:
\begin{equation}
    \Tilde{\mW}_1 = \sqrt{\frac{d_{1}d_{3} \left (  5d_{2}+5d_{3}-1  \right )}{d_{2}^{2} \left (  5d_{1}+5d_{2}-1  \right )}}\cdot \mW_1.
    \label{eq:balanced_initialization}
\end{equation}
This correction ensures that $\mathbb{E}[\|\tilde{\mW}_{1}\tilde{\mW}_{1}^{\top}\|_{F}^{2}] = \mathbb{E}[\|\mW_{2}^{\top}\mW_{2}\|_{F}^{2}]$.

We apply this correction at the start of our procedure to enforce a balanced initialization. We use $\texttt{Haar}(d_{l+1}, d_{1})$ to denote the leading $d_{1}$ columns of a $d_{l+1} \times d_{l+1}$ Haar-distributed \citep{meckes_random_2019} random matrix.
This assumes that $d_{l+1} \geq d_1$, which holds for the architectures for which we experiment with balanced initialization. Note that we only balance the linear section of the networks. The output layer is not modified.

\renewcommand{\algorithmicrequire}{\textbf{Input:}}
\renewcommand{\algorithmicensure}{\textbf{Output:}}

\begin{algorithm}
\caption{Force balancedness between layers of a network}\label{alg:force_balancedness}
\begin{algorithmic}
    \Require{$\mW_{1}$}
    \Ensure{$\tilde{\mW}_{1},\ldots,\tilde{\mW}_{L}$, with $\tilde{\mW}_{l}\tilde{\mW}_{l}^{\top} = \tilde{\mW}_{l+1}^{\top}\tilde{\mW}_{l+1}\quad \forall l \gets 1,\ldots,L-1$}
    \State $\tilde{\mW}_{1} \gets \mW_{1}$ \Comment(Using the above formula)
    \State $\mU_{1}, \mSigma, \mV_1 \gets \texttt{SVD}(\tilde{\mW}_{1})$ \Comment(Using the reduced SVD)
    \For{$l \gets 2$ to $L$}
        \State $\mU_{l} \gets \texttt{Haar}(d_{l+1}, d_{1})$
        \State $\tilde{\mW}_{l} \gets \mU_{l}\mSigma \mU_{l-1}^{\top}$
    \EndFor
    \Return $\tilde{\mW}_{1},\ldots,\tilde{\mW}_{L}$
\end{algorithmic}
\end{algorithm}

To confirm that this initialization is balanced, observe that for $l = 1,\ldots,L-1$, we have
\begin{align*}
    \tilde{\mW}_{l+1}^{\top}\tilde{\mW}_{l+1} &= (\mU_{l+1}\mSigma \mU_{l}^{\top})^{\top} \mU_{l+1}\mSigma U_{l}^{\top} \\
    &= \mU_{l}\mSigma \mU_{l+1}^{\top} \mU_{l+1}\mSigma \mU_{l}^{\top} \\
    &= \mU_{l}\mSigma \mU_{l-1}^{\top} \mU_{l-1}\mSigma \mU_{l}^{\top} \\
    &= \mU_{l}\mSigma \mU_{l-1}^{\top} (\mU_{l}\mSigma \mU_{l-1}^{\top})^{\top} \\
    &= \tilde{\mW}_{l}\tilde{\mW}_{l}^{\top}.
\end{align*}

\section{Additional numerical experiments}

In this section, we present additional numerical experiments.
In \autoref{sec:numerics}, we include results for SGD and GD with and without momentum, learning rank-5 functions with ReLU link functions. We include here results for the Adam optimization algorithm, as well as results for rank-$2$ and rank-$20$ target functions (note that the latter are full rank) and target functions with different link functions. In addition, we conduct experiments on the MNIST \citep{lecun_gradient-based_1998} dataset, to test whether we see similar results when testing on non-synthetic datasets.

\paragraph{Results for Adam.}
Here, we replicate Figures \ref{fig:scaled_alpha_plots} and $\ref{fig:cosine_similarity_though_time}$, replacing SGD by the Adam optimization algorithm. 
The results are included in \autoref{fig:plots_for_adam}. 
Compared to SGD in \autoref{fig:scaled_alpha_plots}, we see that Adam has better alignment after training, when $\lambda = 10^{-3}$, which may be explained by the effect of momentum.
When comparing Adam with SGD with momentum (see \autoref{fig:cosine_similarity_though_time}), we see that the alignment happens for the former at a slightly slower rate than for the latter when $\lambda = 10^{-2}$ (seen in plot (c) of both figures). Since the results for Adam and SGD with momentum are comparable, this suggests that the theoretical results from \autoref{sec:nfa_proof} may be applicable to more general training schemes than  GD (which is a discretization of the gradient flow) and vanilla SGD.  

\begin{figure*}
    \centering
    \begin{subfigure}{0.48\textwidth}
        \centering
        \includegraphics[width=\textwidth, height=0.5\textwidth]{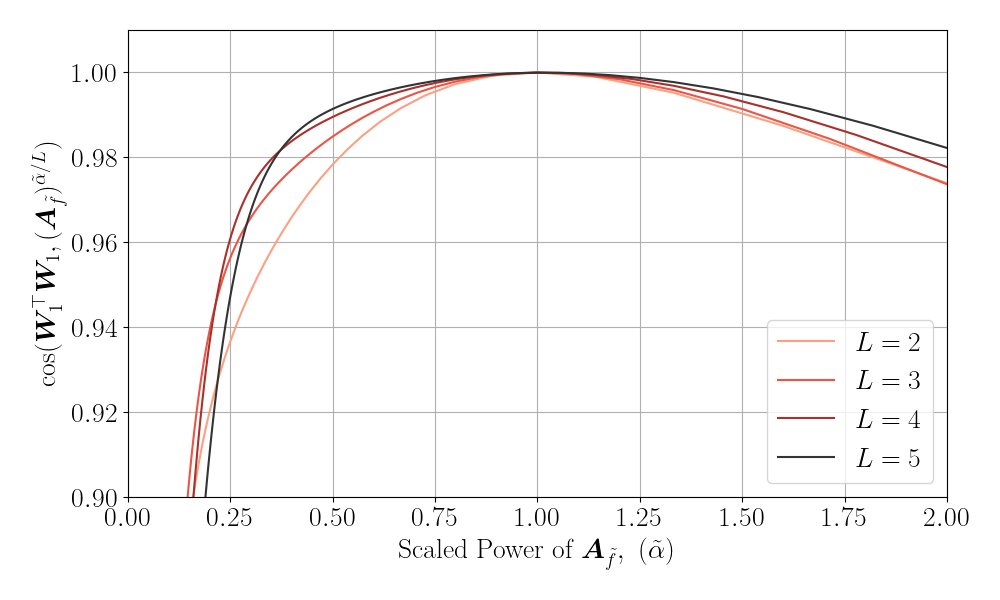}
        \caption{$\lambda = 10^{-2}$}
    \end{subfigure}
    \hfill
    \begin{subfigure}{0.48\textwidth}
        \centering
        \includegraphics[width=\textwidth, height=0.5\textwidth]{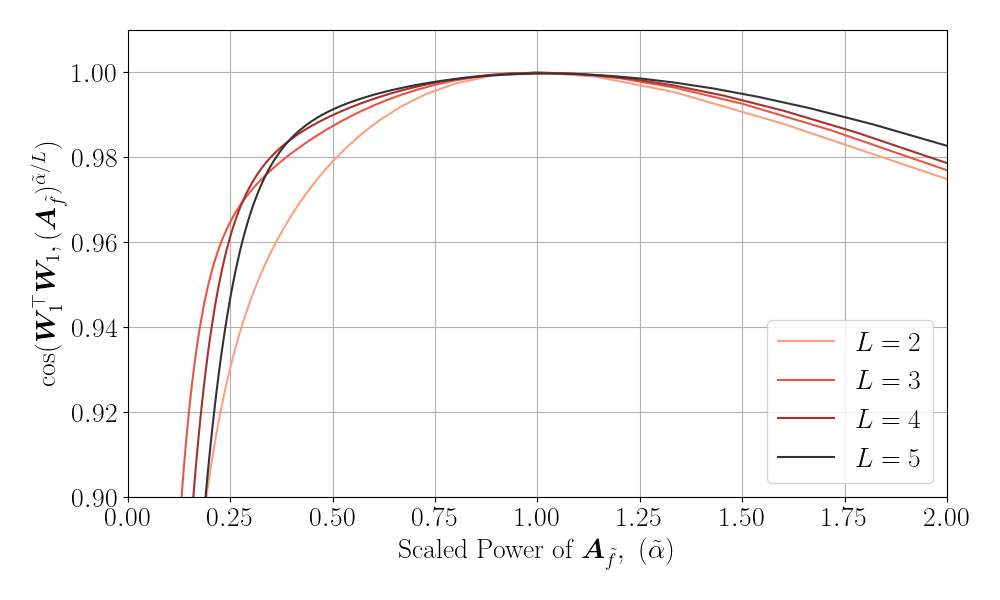}
        \caption{$\lambda = 10^{-3}$}
    \end{subfigure}
    \begin{subfigure}{0.48\textwidth}
        \centering
        \includegraphics[width=\textwidth, height=0.5\textwidth]{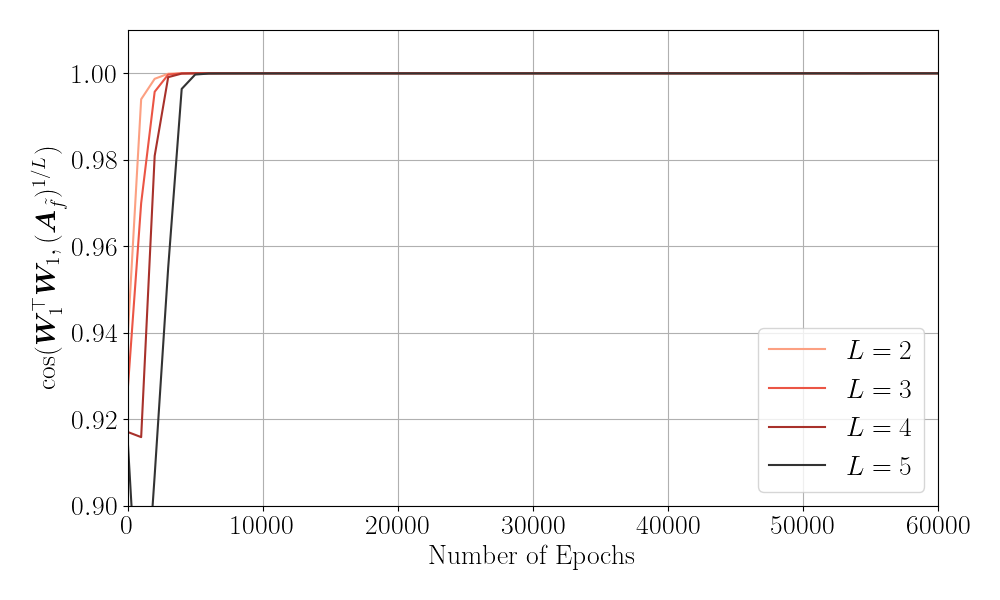}
        \caption{$\lambda = 10^{-2}$}
    \end{subfigure}
    \hfill
    \begin{subfigure}{0.48\textwidth}
        \centering
        \includegraphics[width=\textwidth, height=0.5\textwidth]{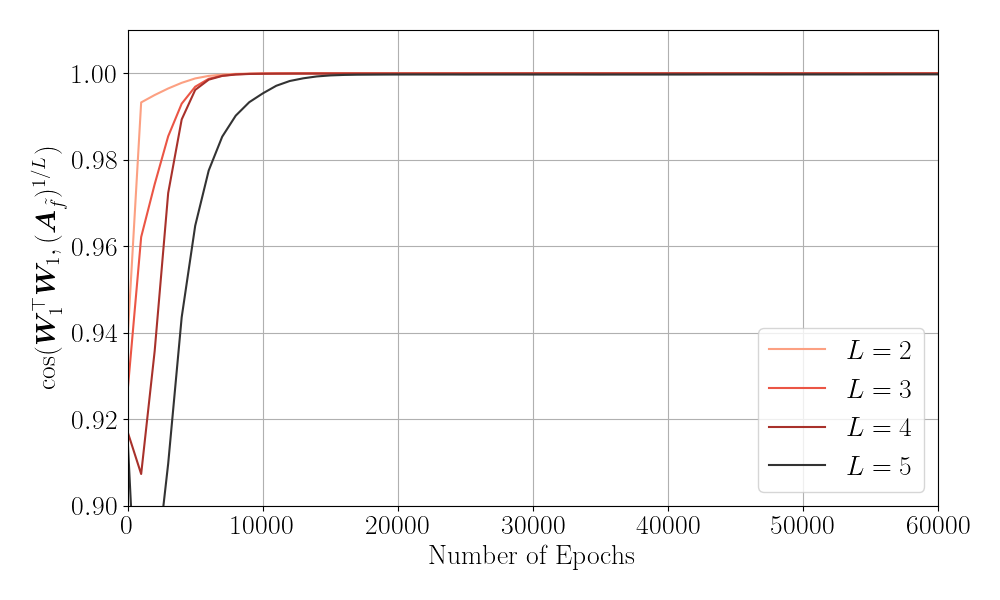}
        \caption{$\lambda = 10^{-3}$}
    \end{subfigure}
    \caption{Illustration of the impact of the weight decay ($\lambda$) on the alignment between $\WonetWone$ and $(\Af)^{1/L}$ at the end of training (Adam). A learning rate of $\eta = 10^{-4}$ was used.}
    \label{fig:plots_for_adam}
\end{figure*}

\paragraph{Changing the link function.}

Instead of the ReLU function, we now show experiments where the data is generated through a Gaussian function, which is given by $g(x) = \exp(-x^{2})$, applied element-wise, in the same manner as the ReLU function in \autoref{sec:numerics}. The results are included in \autoref{fig:cosine_similarity_though_time_gauss}. When training the networks with a weight decay rate of $\lambda = 10^{-2}$, the weights of the network tend to zero, and the network does not learn the target function. Additionally, for both GD and SGD, momentum is required to successfully learn the target function in this instance. We see that the alignment is slower for the Gauss function than the ReLU link function.

\begin{figure*}
    \centering
    \begin{subfigure}{0.325\textwidth}
    \centering
        \includegraphics[width=\textwidth, height=0.5\textwidth]{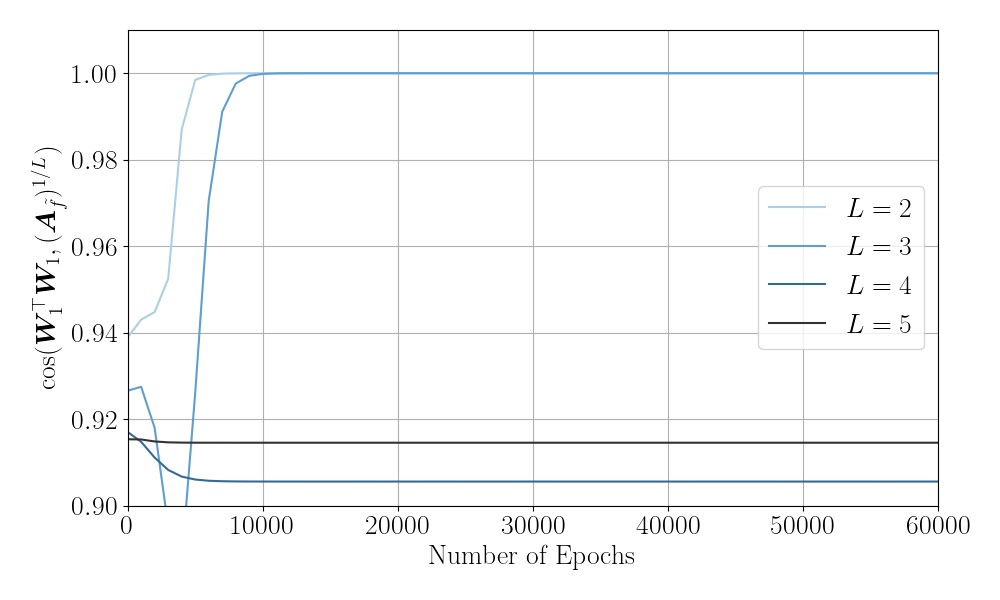}
        \caption{SGD with momentum, $\lambda = 10^{-2}$}
    \end{subfigure}
    \hfill
    \begin{subfigure}{0.325\textwidth}
        \centering
        \includegraphics[width=\textwidth, height=0.5\textwidth]{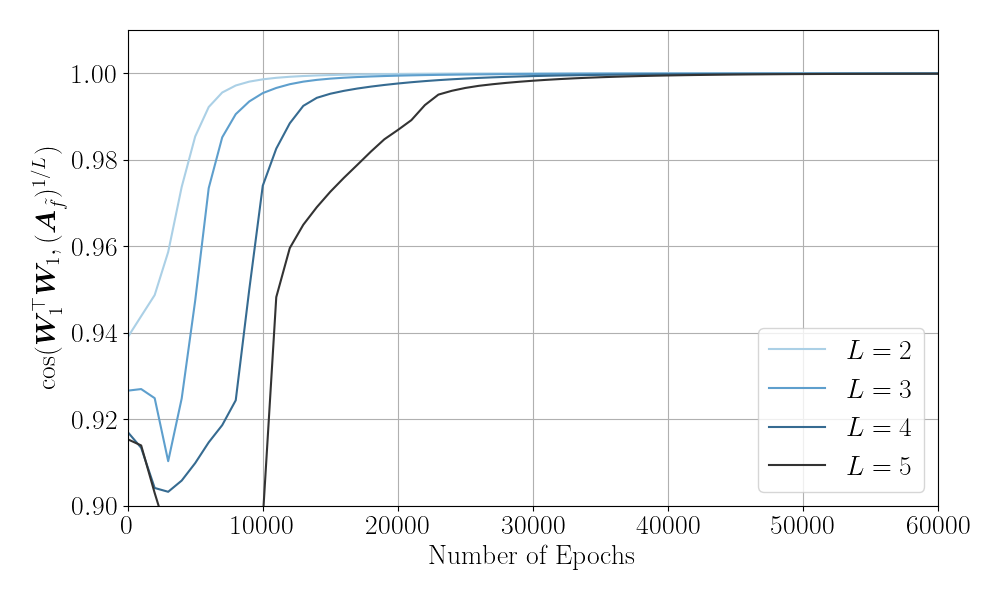}
        \caption{SGD with momentum, $\lambda = 10^{-3}$}
    \end{subfigure}
    \hfill
    \begin{subfigure}{0.325\textwidth}
        \centering
        \includegraphics[width=\textwidth, height=0.5\textwidth]{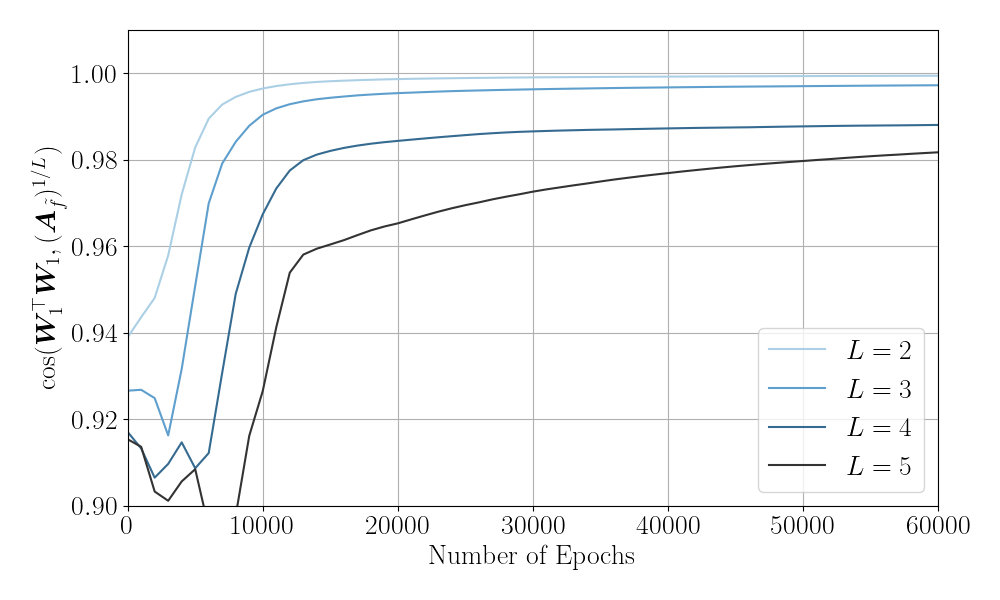}
        \caption{SGD with momentum, $\lambda = 10^{-4}$}
    \end{subfigure}
    \begin{subfigure}{0.325\textwidth}
        \includegraphics[width=\textwidth, height=0.5\textwidth]{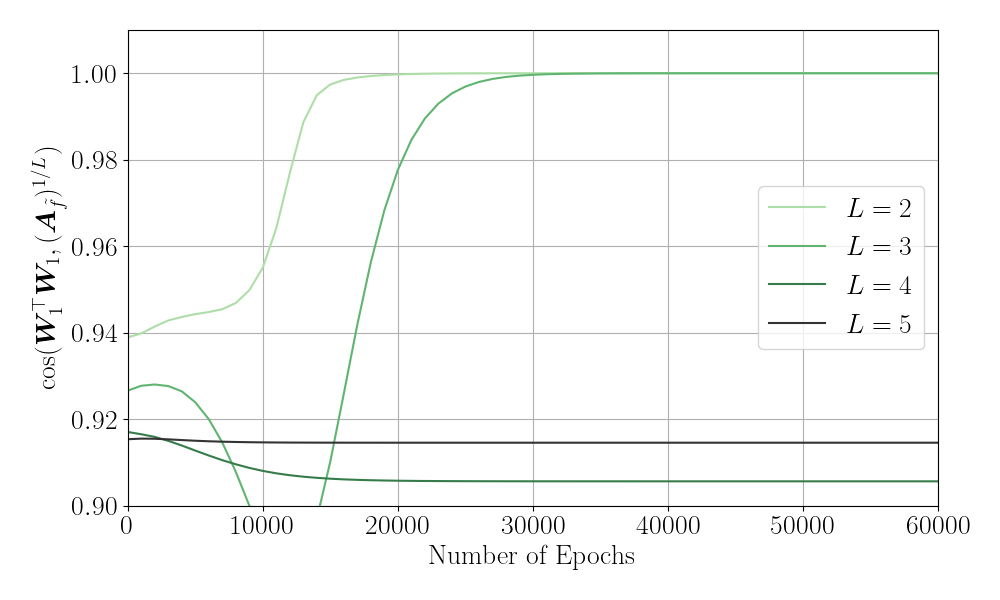}
        \caption{GD with momentum, $\lambda = 10^{-2}$}
    \end{subfigure}
    \hfill
    \begin{subfigure}{0.325\textwidth}
        \centering
        \includegraphics[width=\textwidth, height=0.5\textwidth]{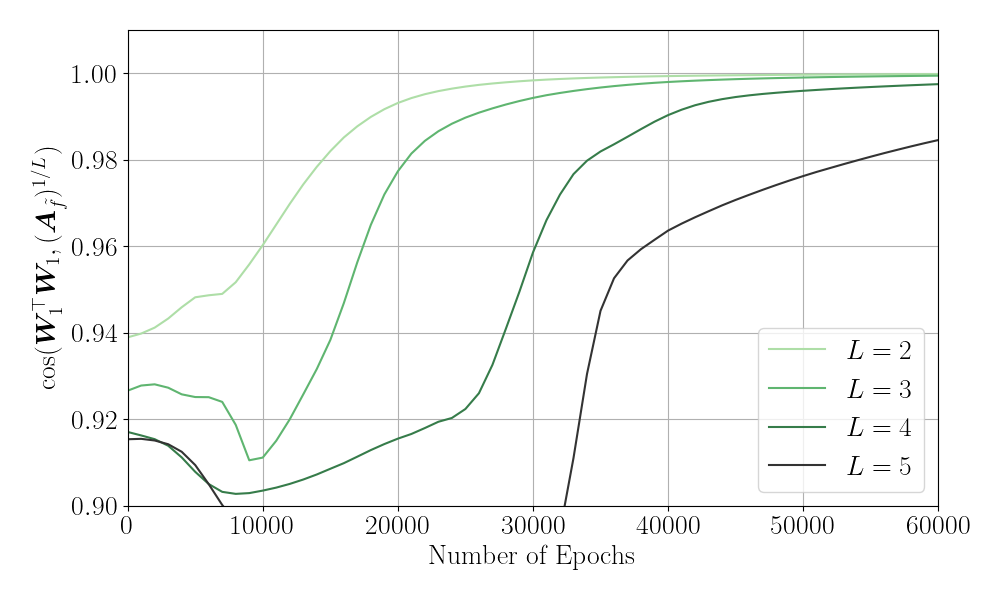}
        \caption{GD with momentum, $\lambda = 10^{-3}$}
    \end{subfigure}
    \hfill
    \begin{subfigure}{0.325\textwidth}
        \centering
        \includegraphics[width=\textwidth, height=0.5\textwidth]{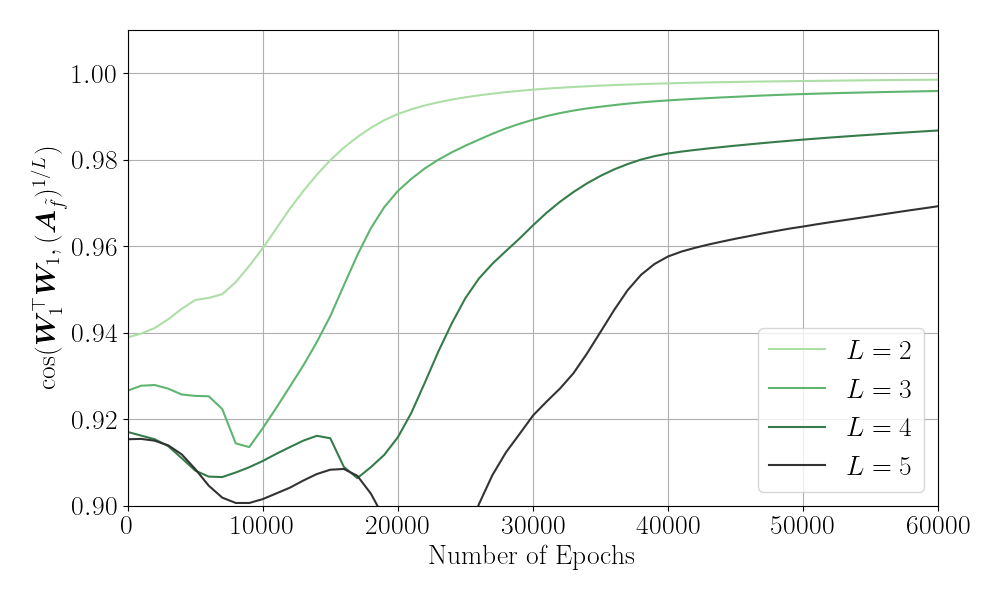}
        \caption{GD with momentum, $\lambda = 10^{-4}$}
    \end{subfigure}
    \caption{Illustration of the impact of the optimization algorithm on the NFA, varying the weight decay, for the Gauss link function. The momentum parameter is set to $\beta = 0.9$. The learning rates for SGD with momentum and GD with momentum are $\eta = 10^{-4}$ and $\eta = 10^{-3}$, respectively.
    When the weight decay is set to $\lambda = 10^{-2}$, the network weights tend to zero and the function is not learnt.}
    \label{fig:cosine_similarity_though_time_gauss}
\end{figure*}

\paragraph{Fully linear networks.}

We now include results for fully linear networks, as these match the theoretical results from \autoref{sec:nfa_proof}.
Here the architecture is given by:
\begin{equation}
    f(\vx) = \mW_{L}\mW_{L-1}\cdot\ldots\cdot\mW_{1}\vx + b_1.\label{eq:architecture_for_linear_experiments}
\end{equation}
As target function, we take
\begin{equation}
f^{*}(\vx) = g(\mA\vx + \vb),    
\end{equation}
where we either let $g$ be the identity function, so that $g(x) = x$, or we let $g : x \mapsto [x]_{+}$ be the ReLU function applied elementwise.
In both cases, $\mA \in \R^{21 \times 20}$ is a rank-$5$ matrix, meaning the output dimension is $21$, but there is low-rank structure in the problem. 
In the case of the ReLU function, the linear network cannot express the target data exactly, but will learn a linear function that best fits the data.

The results are included in \autoref{fig:cosine_similarity_though_time_linear}.  The results appear to be similar when learning the two functions. Furthermore, compared to \autoref{fig:cosine_similarity_though_time} (a), the results are qualitatively similar to the results for networks with a ReLU layer, although with slightly slower convergence.

\begin{figure*}
    \centering
    \begin{subfigure}{0.48\textwidth}
        \centering
        \includegraphics[width=\textwidth, height=0.5\textwidth]{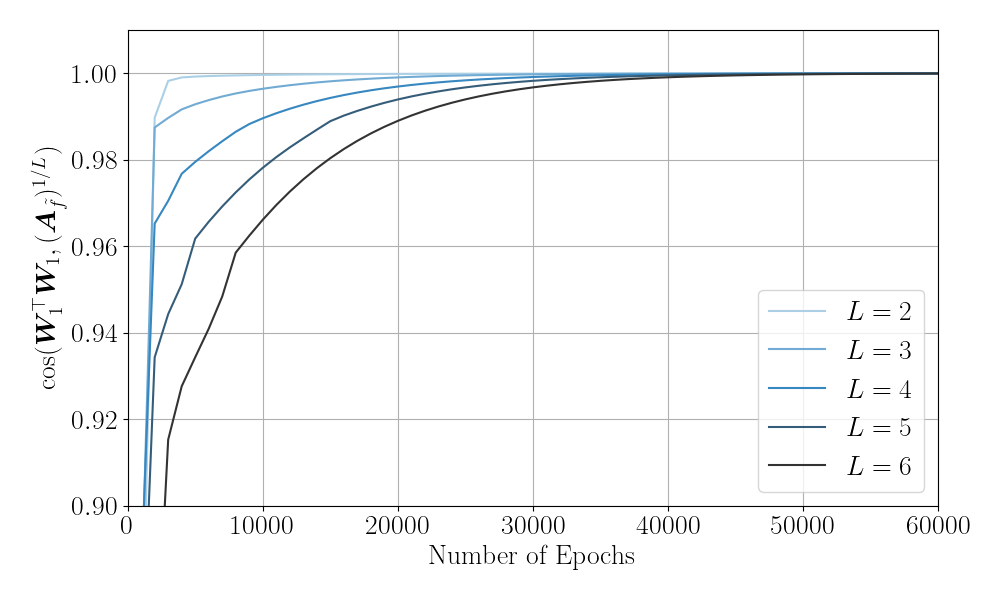}
        \caption{$g(x) = x$, $\lambda = 10^{-2}$}
    \end{subfigure}
    \hfill
    \begin{subfigure}{0.48\textwidth}
        \centering
        \includegraphics[width=\textwidth, height=0.5\textwidth]{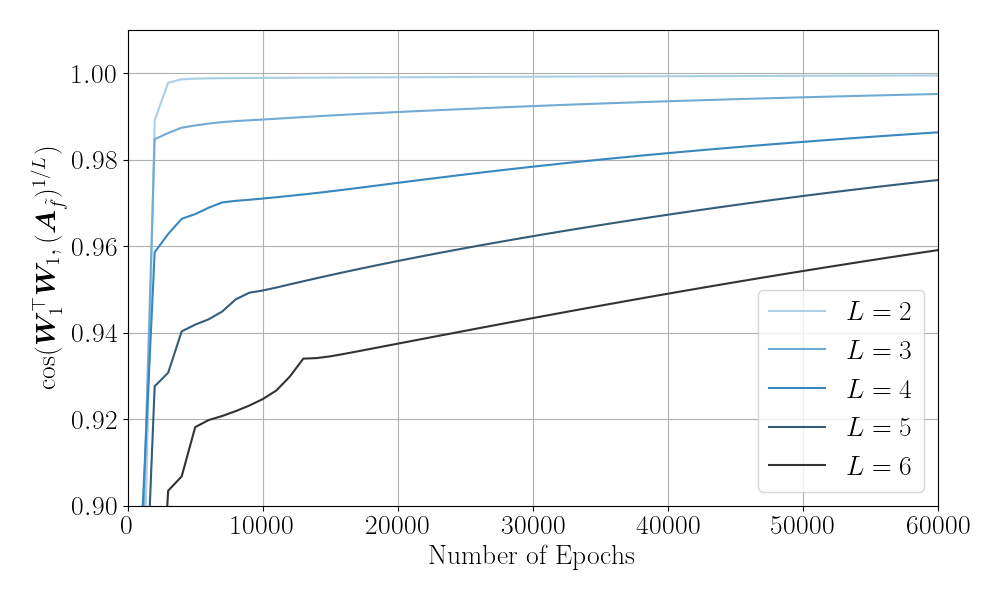}
        \caption{$g(x) = x$, $\lambda = 10^{-3}$}
    \end{subfigure}
    \begin{subfigure}{0.48\textwidth}
        \centering
        \includegraphics[width=\textwidth, height=0.5\textwidth]{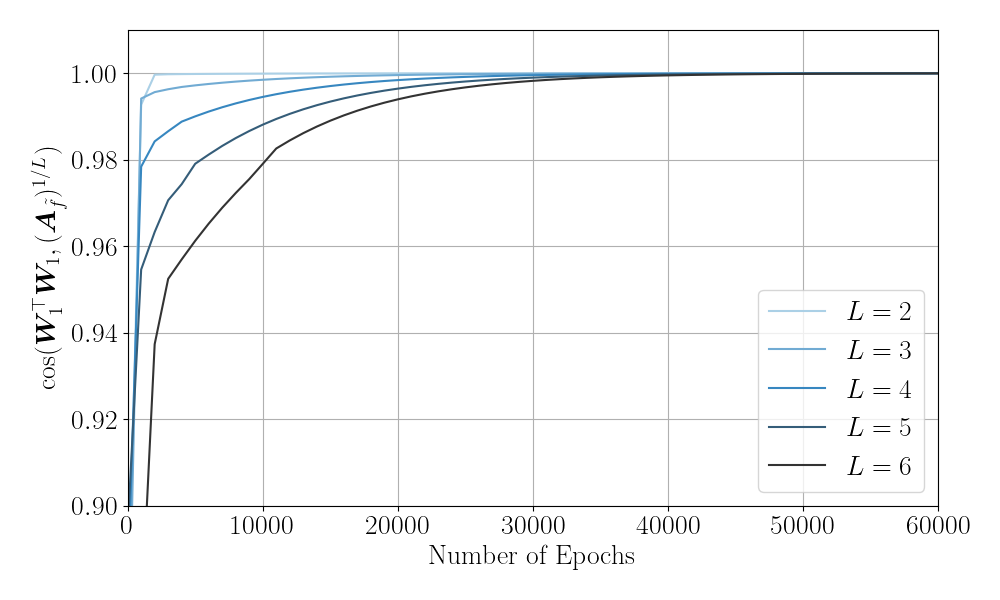}
        \caption{$g(x) = [x]_{+}$, $\lambda = 10^{-2}$}
    \end{subfigure}
    \hfill
    \begin{subfigure}{0.48\textwidth}
        \centering
        \includegraphics[width=\textwidth, height=0.5\textwidth]{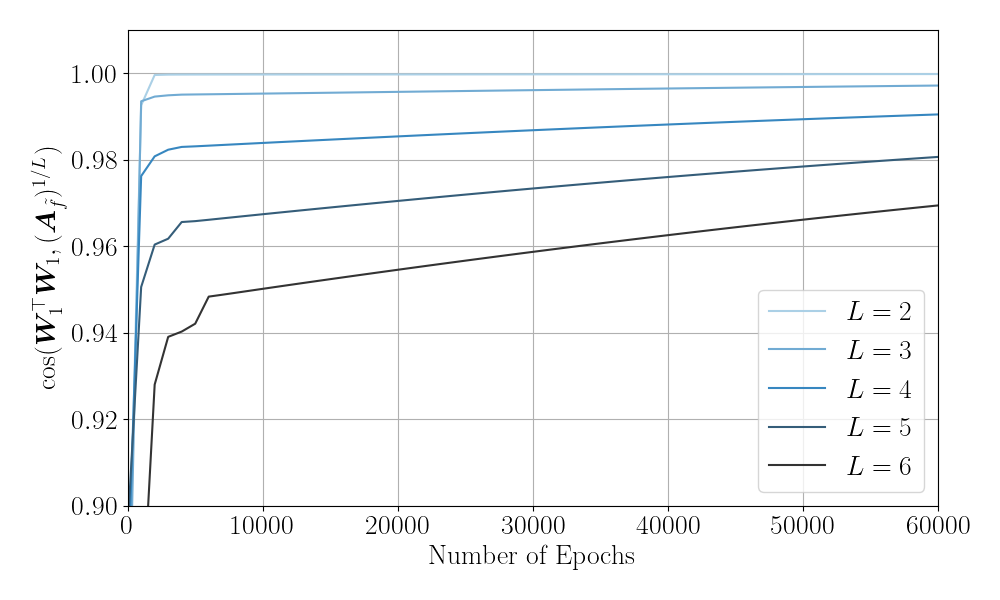}
        \caption{$g(x) = [x]_{+}$, $\lambda = 10^{-3}$}
    \end{subfigure}
    \caption{Testing the NFA on fully linear networks. We throughout use SGD with $\eta = 10^{-3}$. We vary the objective function between the top and bottom rows.}
    \label{fig:cosine_similarity_though_time_linear}
\end{figure*}

\paragraph{Experiments on MNIST.}

In addition to our experiments on synthetic data, we also present results on the MNIST dataset of handwritten digits. Note that this is a classification task, in contrast with the regression tasks that were addressed in \autoref{sec:numerics}, as such the loss function and architectures used shall differ from our earlier experiments.
We use the same optimization algorithms and hyperparameters as our other numerical experiments. However, we now run our experiments for 200 epochs, as we observe that this is typically sufficient for the training on this data set. In our experiments, training accuracy for  networks trained with Adam is typically $\geq 97\%$, whilst the training accuracy for those trained with SGD is typically $\geq 95\%$. We see that the weight matrices align more quickly when Adam is used than they do when SGD with momentum  is employed and that the rate depends on the weight decay parameter $\lambda$. We include the results in \autoref{fig:cosine_similarity_though_time_mnist}.

As we obtain similar results when learning to classify the MNIST dataset as we do for the regression task on synthetic data, this suggests that our theoretical results for the NFA on deep networks in \autoref{sec:nfa_proof} may be applicable to a broader range of tasks than those considered in \autoref{sec:numerics}.

\begin{figure*}
    \centering
    \begin{subfigure}{0.48\textwidth}
        \centering
        \includegraphics[width=\textwidth, height=0.5\textwidth]{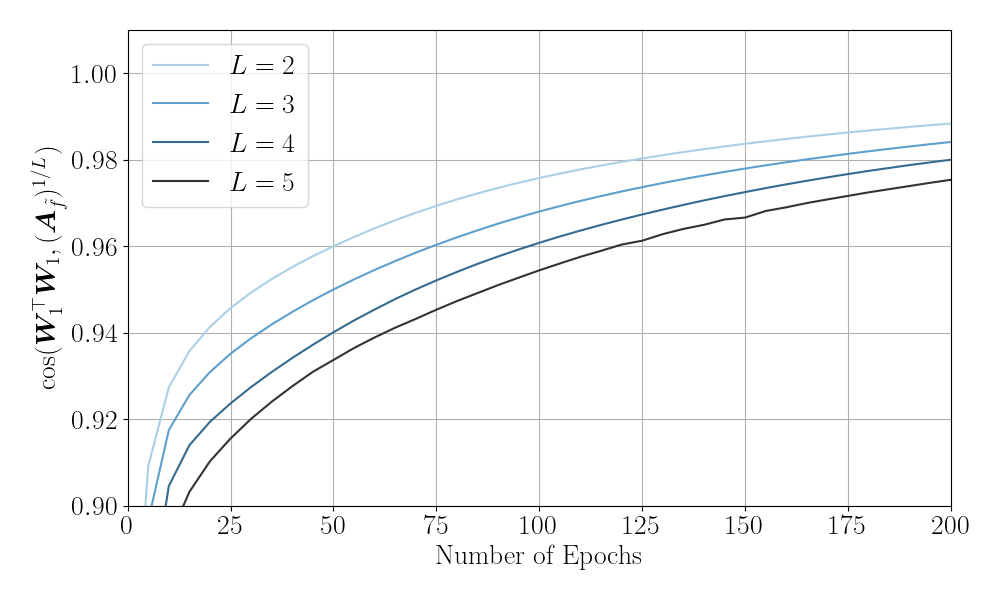}
        \caption{SGD with momentum, $\lambda  = 10^{-3}$}
    \end{subfigure}
    \hfill
    \begin{subfigure}{0.48\textwidth}
        \centering
        \includegraphics[width=\textwidth, height=0.5\textwidth]{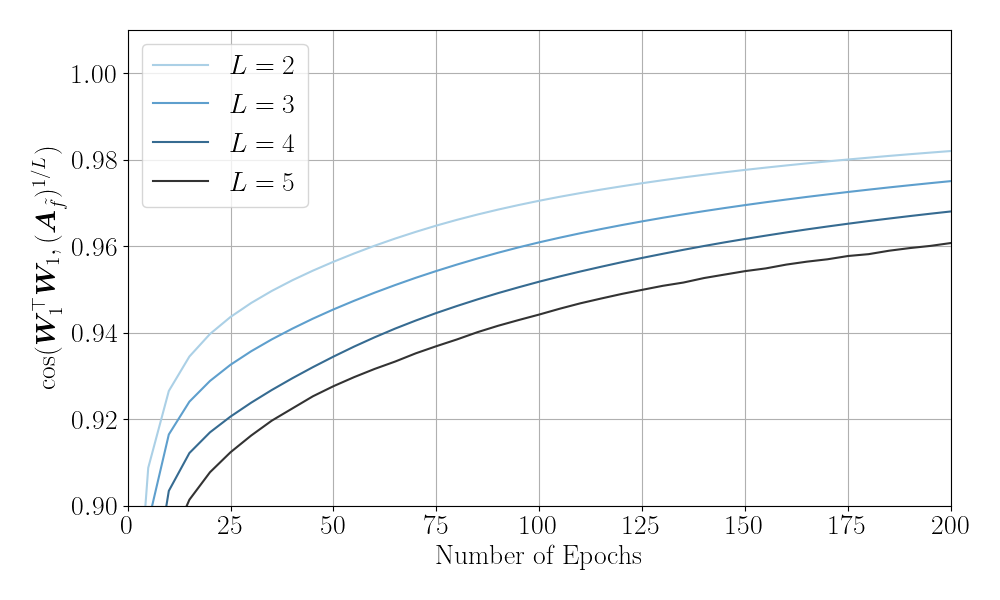}
        \caption{SGD with momentum, $\lambda  = 10^{-4}$}
    \end{subfigure}
    \begin{subfigure}{0.48\textwidth}
    \includegraphics[width=\textwidth, height=0.5\textwidth]{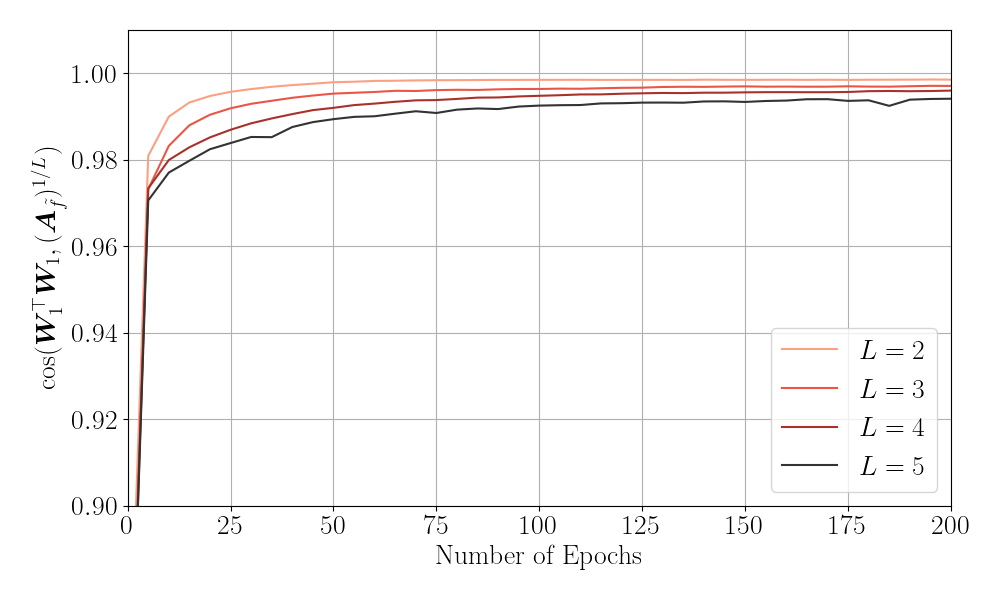}
        \caption{Adam, $\lambda  = 10^{-3}$}
    \end{subfigure}
    \hfill
    \begin{subfigure}{0.48\textwidth}
        \centering
        \includegraphics[width=\textwidth, height=0.5\textwidth]{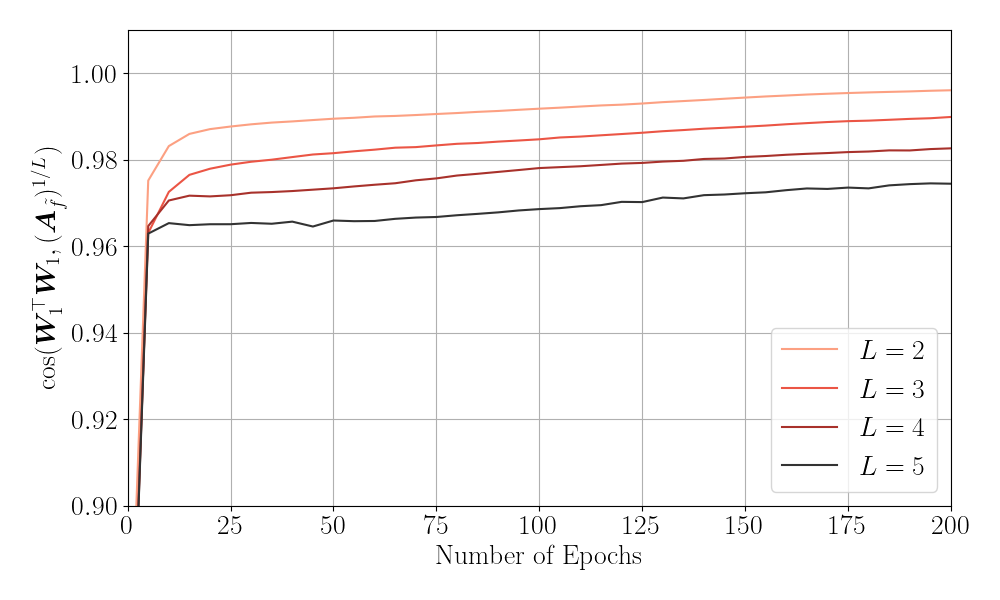}
        \caption{Adam, $\lambda  = 10^{-4}$}
    \end{subfigure}
    \caption{Illustration of the impact of the optimization algorithm on the NFA, on the MNIST dataset.}
    \label{fig:cosine_similarity_though_time_mnist}
\end{figure*}

\clearpage
\newpage

\subsection{Tables of results}

We include tables of results for all of the optimization algorithms mentioned in \autoref{sec:numerics}. We additionally include our results for the Adam optimization algorithm.
In these tables, we include the cosine similarity between $\WonetWone$ and $(\AfTilde)^{1/L}$ after training for the number of initial linear layers $L$ of the network.
We include results that have no label noise ($\sigma = 0$) as well as those that have label noise of $\sigma = 1$. In each of these tables we use $N=2048$ data points. gdm and sgdm refer to GD and SGD with momentum parameter $\beta = 0.9$, respectively. In these tables, we round all results to 2 decimal places.

\begin{table}[h]
    \centering
    \caption{Rank 2}
\begin{tabular}{rrllllll}
\toprule
sigma & Linear Layers & lambda & adam & gd & gdm & sgd & sgdm \\
\midrule
0 & 2 & $10^{-5}$ & $\textbf{1.00}$ & 0.99 & $\textbf{1.00}$ & 0.99 & $\textbf{1.00}$ \\
0 & 2 & $10^{-4}$ & $\textbf{1.00}$ & 0.99 & $\textbf{1.00}$ & 0.99 & $\textbf{1.00}$ \\
0 & 2 & $10^{-3}$ & $\textbf{1.00}$ & 0.99 & $\textbf{1.00}$ & $\textbf{1.00}$ & $\textbf{1.00}$ \\
0 & 2 & $10^{-2}$ & $\textbf{1.00}$ & $\textbf{1.00}$ & $\textbf{1.00}$ & $\textbf{1.00}$ & $\textbf{1.00}$ \\
0 & 3 & $10^{-5}$ & 0.98 & 0.97 & $\textbf{0.99}$ & 0.97 & 0.98 \\
0 & 3 & $10^{-4}$ & $\textbf{1.00}$ & 0.97 & 0.99 & 0.97 & 0.99 \\
0 & 3 & $10^{-3}$ & $\textbf{1.00}$ & 0.98 & $\textbf{1.00}$ & 0.99 & $\textbf{1.00}$ \\
0 & 3 & $10^{-2}$ & $\textbf{1.00}$ & $\textbf{1.00}$ & $\textbf{1.00}$ & $\textbf{1.00}$ & $\textbf{1.00}$ \\
0 & 4 & $10^{-5}$ & $\textbf{0.99}$ & 0.94 & 0.97 & 0.95 & 0.96 \\
0 & 4 & $10^{-4}$ & $\textbf{1.00}$ & 0.94 & 0.97 & 0.95 & 0.98 \\
0 & 4 & $10^{-3}$ & $\textbf{1.00}$ & 0.95 & 0.99 & 0.97 & $\textbf{1.00}$ \\
0 & 4 & $10^{-2}$ & $\textbf{1.00}$ & 0.99 & $\textbf{1.00}$ & $\textbf{1.00}$ & $\textbf{1.00}$ \\
0 & 5 & $10^{-5}$ & $\textbf{0.96}$ & 0.91 & $\textbf{0.96}$ & 0.92 & 0.94 \\
0 & 5 & $10^{-4}$ & $\textbf{0.99}$ & 0.92 & 0.97 & 0.92 & 0.96 \\
0 & 5 & $10^{-3}$ & $\textbf{1.00}$ & 0.92 & 0.99 & 0.95 & $\textbf{1.00}$ \\
0 & 5 & $10^{-2}$ & $\textbf{1.00}$ & 0.99 & $\textbf{1.00}$ & $\textbf{1.00}$ & $\textbf{1.00}$ \\
1 & 2 & $10^{-5}$ & $\textbf{1.00}$ & 0.99 & $\textbf{1.00}$ & $\textbf{1.00}$ & $\textbf{1.00}$ \\
1 & 2 & $10^{-4}$ & $\textbf{1.00}$ & 0.99 & $\textbf{1.00}$ & $\textbf{1.00}$ & $\textbf{1.00}$ \\
1 & 2 & $10^{-3}$ & $\textbf{1.00}$ & 0.99 & $\textbf{1.00}$ & $\textbf{1.00}$ & $\textbf{1.00}$ \\
1 & 2 & $10^{-2}$ & $\textbf{1.00}$ & $\textbf{1.00}$ & $\textbf{1.00}$ & $\textbf{1.00}$ & $\textbf{1.00}$ \\
1 & 3 & $10^{-5}$ & 0.99 & 0.98 & $\textbf{1.00}$ & 0.99 & 0.98 \\
1 & 3 & $10^{-4}$ & $\textbf{1.00}$ & 0.98 & $\textbf{1.00}$ & 0.99 & 0.99 \\
1 & 3 & $10^{-3}$ & $\textbf{1.00}$ & 0.98 & $\textbf{1.00}$ & 0.99 & $\textbf{1.00}$ \\
1 & 3 & $10^{-2}$ & $\textbf{1.00}$ & $\textbf{1.00}$ & $\textbf{1.00}$ & $\textbf{1.00}$ & $\textbf{1.00}$ \\
1 & 4 & $10^{-5}$ & 0.97 & 0.95 & $\textbf{0.99}$ & 0.98 & 0.92 \\
1 & 4 & $10^{-4}$ & $\textbf{0.99}$ & 0.95 & $\textbf{0.99}$ & 0.98 & 0.93 \\
1 & 4 & $10^{-3}$ & $\textbf{1.00}$ & 0.96 & $\textbf{1.00}$ & 0.98 & 0.98 \\
1 & 4 & $10^{-2}$ & $\textbf{1.00}$ & 0.99 & $\textbf{1.00}$ & $\textbf{1.00}$ & $\textbf{1.00}$ \\
1 & 5 & $10^{-5}$ & 0.90 & 0.92 & $\textbf{0.98}$ & 0.96 & 0.85 \\
1 & 5 & $10^{-4}$ & 0.93 & 0.92 & $\textbf{0.98}$ & 0.96 & 0.88 \\
1 & 5 & $10^{-3}$ & 0.98 & 0.93 & $\textbf{0.99}$ & 0.97 & 0.94 \\
1 & 5 & $10^{-2}$ & $\textbf{1.00}$ & 0.99 & $\textbf{1.00}$ & $\textbf{1.00}$ & 0.99 \\
\bottomrule
\end{tabular}
\end{table}

\begin{table}[h]
    \centering
    \caption{Rank 5}
\begin{tabular}{rrllllll}
\toprule
sigma & Linear Layers & lambda & adam & gd & gdm & sgd & sgdm \\
\midrule
0 & 2 & $10^{-5}$ & $\textbf{1.00}$ & $\textbf{1.00}$ & $\textbf{1.00}$ & $\textbf{1.00}$ & $\textbf{1.00}$ \\
0 & 2 & $10^{-4}$ & $\textbf{1.00}$ & $\textbf{1.00}$ & $\textbf{1.00}$ & $\textbf{1.00}$ & $\textbf{1.00}$ \\
0 & 2 & $10^{-3}$ & $\textbf{1.00}$ & $\textbf{1.00}$ & $\textbf{1.00}$ & $\textbf{1.00}$ & $\textbf{1.00}$ \\
0 & 2 & $10^{-2}$ & $\textbf{1.00}$ & $\textbf{1.00}$ & $\textbf{1.00}$ & $\textbf{1.00}$ & $\textbf{1.00}$ \\
0 & 3 & $10^{-5}$ & $\textbf{0.99}$ & $\textbf{0.99}$ & $\textbf{0.99}$ & $\textbf{0.99}$ & $\textbf{0.99}$ \\
0 & 3 & $10^{-4}$ & $\textbf{1.00}$ & 0.99 & 0.99 & 0.99 & 0.99 \\
0 & 3 & $10^{-3}$ & $\textbf{1.00}$ & 0.99 & $\textbf{1.00}$ & $\textbf{1.00}$ & $\textbf{1.00}$ \\
0 & 3 & $10^{-2}$ & $\textbf{1.00}$ & $\textbf{1.00}$ & $\textbf{1.00}$ & $\textbf{1.00}$ & $\textbf{1.00}$ \\
0 & 4 & $10^{-5}$ & 0.97 & $\textbf{0.98}$ & $\textbf{0.98}$ & $\textbf{0.98}$ & 0.97 \\
0 & 4 & $10^{-4}$ & $\textbf{1.00}$ & 0.98 & 0.98 & 0.98 & 0.98 \\
0 & 4 & $10^{-3}$ & $\textbf{1.00}$ & 0.98 & $\textbf{1.00}$ & 0.99 & $\textbf{1.00}$ \\
0 & 4 & $10^{-2}$ & $\textbf{1.00}$ & $\textbf{1.00}$ & $\textbf{1.00}$ & $\textbf{1.00}$ & $\textbf{1.00}$ \\
0 & 5 & $10^{-5}$ & 0.95 & 0.96 & 0.96 & $\textbf{0.97}$ & 0.95 \\
0 & 5 & $10^{-4}$ & $\textbf{0.99}$ & 0.96 & 0.96 & 0.97 & 0.97 \\
0 & 5 & $10^{-3}$ & $\textbf{1.00}$ & 0.97 & 0.99 & 0.98 & $\textbf{1.00}$ \\
0 & 5 & $10^{-2}$ & $\textbf{1.00}$ & $\textbf{1.00}$ & $\textbf{1.00}$ & $\textbf{1.00}$ & $\textbf{1.00}$ \\
1 & 2 & $10^{-5}$ & $\textbf{1.00}$ & $\textbf{1.00}$ & $\textbf{1.00}$ & $\textbf{1.00}$ & $\textbf{1.00}$ \\
1 & 2 & $10^{-4}$ & $\textbf{1.00}$ & $\textbf{1.00}$ & $\textbf{1.00}$ & $\textbf{1.00}$ & $\textbf{1.00}$ \\
1 & 2 & $10^{-3}$ & $\textbf{1.00}$ & $\textbf{1.00}$ & $\textbf{1.00}$ & $\textbf{1.00}$ & $\textbf{1.00}$ \\
1 & 2 & $10^{-2}$ & $\textbf{1.00}$ & $\textbf{1.00}$ & $\textbf{1.00}$ & $\textbf{1.00}$ & $\textbf{1.00}$ \\
1 & 3 & $10^{-5}$ & 0.99 & 0.99 & $\textbf{1.00}$ & 0.99 & 0.97 \\
1 & 3 & $10^{-4}$ & 0.99 & 0.99 & $\textbf{1.00}$ & $\textbf{1.00}$ & 0.98 \\
1 & 3 & $10^{-3}$ & $\textbf{1.00}$ & 0.99 & $\textbf{1.00}$ & $\textbf{1.00}$ & $\textbf{1.00}$ \\
1 & 3 & $10^{-2}$ & $\textbf{1.00}$ & $\textbf{1.00}$ & $\textbf{1.00}$ & $\textbf{1.00}$ & $\textbf{1.00}$ \\
1 & 4 & $10^{-5}$ & 0.97 & 0.98 & $\textbf{0.99}$ & $\textbf{0.99}$ & 0.94 \\
1 & 4 & $10^{-4}$ & 0.98 & 0.98 & 0.98 & $\textbf{0.99}$ & 0.94 \\
1 & 4 & $10^{-3}$ & 0.99 & 0.99 & $\textbf{1.00}$ & 0.99 & 0.98 \\
1 & 4 & $10^{-2}$ & $\textbf{1.00}$ & $\textbf{1.00}$ & $\textbf{1.00}$ & $\textbf{1.00}$ & $\textbf{1.00}$ \\
1 & 5 & $10^{-5}$ & 0.94 & 0.97 & $\textbf{0.98}$ & $\textbf{0.98}$ & 0.88 \\
1 & 5 & $10^{-4}$ & 0.96 & 0.97 & $\textbf{0.99}$ & 0.98 & 0.89 \\
1 & 5 & $10^{-3}$ & 0.99 & 0.97 & $\textbf{1.00}$ & 0.99 & 0.96 \\
1 & 5 & $10^{-2}$ & $\textbf{1.00}$ & $\textbf{1.00}$ & $\textbf{1.00}$ & $\textbf{1.00}$ & 0.99 \\
\bottomrule
\end{tabular}
\end{table}

\begin{table}[h]
    \centering
    \caption{Rank 20. \textbf{nan} indicate that the training failed resulting in nan values in the weight matrices.}
\begin{tabular}{rrllllll}
\toprule
sigma & Linear Layers & lambda & adam & gd & gdm & sgd & sgdm \\
\midrule
0 & 2 & $10^{-5}$ & 0.99 & $\textbf{1.00}$ & $\textbf{1.00}$ & $\textbf{1.00}$ & 0.99 \\
0 & 2 & $10^{-4}$ & $\textbf{1.00}$ & $\textbf{1.00}$ & $\textbf{1.00}$ & $\textbf{1.00}$ & $\textbf{1.00}$ \\
0 & 2 & $10^{-3}$ & $\textbf{1.00}$ & $\textbf{1.00}$ & $\textbf{1.00}$ & $\textbf{1.00}$ & $\textbf{1.00}$ \\
0 & 2 & $10^{-2}$ & $\textbf{1.00}$ & $\textbf{1.00}$ & $\textbf{1.00}$ & $\textbf{1.00}$ & $\textbf{1.00}$ \\
0 & 3 & $10^{-5}$ & 0.97 & $\textbf{1.00}$ & 0.98 & 0.99 & 0.92 \\
0 & 3 & $10^{-4}$ & $\textbf{1.00}$ & $\textbf{1.00}$ & 0.98 & 0.99 & 0.95 \\
0 & 3 & $10^{-3}$ & $\textbf{1.00}$ & $\textbf{1.00}$ & $\textbf{1.00}$ & $\textbf{1.00}$ & $\textbf{1.00}$ \\
0 & 3 & $10^{-2}$ & $\textbf{1.00}$ & $\textbf{1.00}$ & $\textbf{1.00}$ & $\textbf{1.00}$ & $\textbf{1.00}$ \\
0 & 4 & $10^{-5}$ & 0.94 & $\textbf{0.97}$ & 0.84 & $\textbf{0.97}$ & 0.88 \\
0 & 4 & $10^{-4}$ & $\textbf{0.98}$ & 0.97 & 0.87 & 0.97 & 0.92 \\
0 & 4 & $10^{-3}$ & $\textbf{1.00}$ & 0.98 & 0.97 & 0.98 & $\textbf{1.00}$ \\
0 & 4 & $10^{-2}$ & $\textbf{1.00}$ & 0.99 & $\textbf{1.00}$ & $\textbf{1.00}$ & 0.99 \\
0 & 5 & $10^{-5}$ & 0.94 & 0.97 & $\textbf{nan}$ & 0.94 & 0.83 \\
0 & 5 & $10^{-4}$ & 0.95 & 0.97 & $\textbf{nan}$ & 0.95 & 0.87 \\
0 & 5 & $10^{-3}$ & $\textbf{0.98}$ & 0.97 & 0.89 & 0.97 & 0.92 \\
0 & 5 & $10^{-2}$ & $\textbf{0.99}$ & 0.98 & $\textbf{0.99}$ & $\textbf{0.99}$ & 0.93 \\
1 & 2 & $10^{-5}$ & 0.98 & $\textbf{1.00}$ & $\textbf{1.00}$ & $\textbf{1.00}$ & 0.98 \\
1 & 2 & $10^{-4}$ & 0.98 & $\textbf{1.00}$ & $\textbf{1.00}$ & $\textbf{1.00}$ & 0.99 \\
1 & 2 & $10^{-3}$ & $\textbf{1.00}$ & $\textbf{1.00}$ & $\textbf{1.00}$ & $\textbf{1.00}$ & $\textbf{1.00}$ \\
1 & 2 & $10^{-2}$ & $\textbf{1.00}$ & $\textbf{1.00}$ & $\textbf{1.00}$ & $\textbf{1.00}$ & $\textbf{1.00}$ \\
1 & 3 & $10^{-5}$ & 0.95 & $\textbf{1.00}$ & 0.98 & 0.99 & 0.89 \\
1 & 3 & $10^{-4}$ & 0.96 & $\textbf{1.00}$ & 0.99 & 0.99 & 0.91 \\
1 & 3 & $10^{-3}$ & 0.99 & $\textbf{1.00}$ & $\textbf{1.00}$ & $\textbf{1.00}$ & 0.98 \\
1 & 3 & $10^{-2}$ & $\textbf{1.00}$ & $\textbf{1.00}$ & $\textbf{1.00}$ & $\textbf{1.00}$ & $\textbf{1.00}$ \\
1 & 4 & $10^{-5}$ & 0.92 & $\textbf{0.97}$ & 0.84 & $\textbf{0.97}$ & 0.86 \\
1 & 4 & $10^{-4}$ & 0.94 & 0.96 & 0.87 & $\textbf{0.97}$ & 0.87 \\
1 & 4 & $10^{-3}$ & $\textbf{0.98}$ & 0.96 & 0.95 & $\textbf{0.98}$ & 0.94 \\
1 & 4 & $10^{-2}$ & $\textbf{1.00}$ & 0.99 & $\textbf{1.00}$ & $\textbf{1.00}$ & 0.98 \\
1 & 5 & $10^{-5}$ & 0.94 & 0.97 & $\textbf{nan}$ & 0.94 & 0.83 \\
1 & 5 & $10^{-4}$ & 0.94 & 0.97 & $\textbf{nan}$ & 0.95 & 0.83 \\
1 & 5 & $10^{-3}$ & 0.97 & 0.97 & $\textbf{nan}$ & 0.97 & 0.89 \\
1 & 5 & $10^{-2}$ & $\textbf{0.99}$ & 0.98 & 0.98 & $\textbf{0.99}$ & 0.93 \\
\bottomrule
\end{tabular}
\end{table}

\clearpage
\newpage

\section{Proof of \autoref{lem:W1_rescaling}}
\label{sec:lem:W1_rescaling_proof}

We state the following standard result on moments of the uniform distribution. 

\begin{lemma}
    Suppose $\mX \sim U(-1, 1)$, then $\mathbb{E}[\mX^{k}]$ is given by:
    \begin{equation*}
        \mathbb{E}[\mX^{k}] =
        \begin{cases}
            0 &\text{for k is odd} \\
            \frac{1}{k+1} &\text{for k even}
        \end{cases}
    \end{equation*}
    \label{lem:uniform_moments}
\end{lemma}

\begin{lemma}\label{lem:W1_recaling_helper}
    Suppose that $\mA \in \R^{m \times n}$ has entries drawn $a_{ij} \sim U(-1, 1)$, then $\mathbb{E}[\|\mA^{\top}\mA\|_{F}^{2}] = mn \left ( \frac{1}{5} + \frac{m+n-2}{9}  \right )$. 
\end{lemma}
\begin{proof}
   From definitions, we have
    \begin{equation*}
        \|\mA^{\top}\mA\|_{F}^{2}  = \sum_{i=1}^{n}\sum_{j=1}^{n}(\mA^{\top}\mA)_{ij}^{2} 
        = \sum_{i=1}^{n}\sum_{j=1}^{n}\left ( \sum_{k=1}^{n}a_{ki}a_{kj} \right )^{2} 
        = \sum_{i=1}^{n}\sum_{j=1}^{n}\sum_{k=1}^{m}\sum_{l=1}^{m}a_{ki}a_{kj}a_{li}a_{lj},
    \end{equation*}
    and hence by linearity of expectation, we have that
    \begin{equation*}
        \mathbb{E}[\|\mA^{\top}\mA\|_{F}^{2}] = \sum_{i=1}^{n}\sum_{j=1}^{n}\sum_{k=1}^{m}\sum_{l=1}^{m} \mathbb{E}[a_{ki}a_{kj}a_{li}a_{lj}].
    \end{equation*}
    We calculate the expectation of each term in this sum depending on various cases on $i,j,k,l$. We throughout use the result of \autoref{lem:uniform_moments}.
    \begin{enumerate}
    \item $i = j$ and $k = l$, \textbf{($mn$ such terms)} :
        \begin{align*}
            \mathbb{E}[a_{ki}a_{kj}a_{li}a_{lj}] = \mathbb{E}[a_{ki}^{4}] = \frac{1}{5} 
        \end{align*}
    \item $i = j$, but, $k \neq l$ \textbf{($m(m-1)n$ such terms)}:
        \begin{align*}
            \mathbb{E}[a_{ki}a_{kj}a_{li}a_{lj}] = \mathbb{E}[a_{ki}^{2}a_{li}^{2}] = \mathbb{E}[a_{ki}^{2}]\mathbb{E}[a_{li}^{2}] =  \frac{1}{9},
        \end{align*}
        where for the third inequality, we use the independence of elements
    \item $k = l$, but, $i \neq j$ \textbf{($mn(n-1)$ such terms)}:
        \begin{align*}
            \mathbb{E}[a_{ki}a_{kj}a_{li}a_{lj}] = \mathbb{E}[a_{ki}^{2}a_{kj}^{2}] = \mathbb{E}[a_{ki}^{2}]\mathbb{E}[a_{kj}^{2}] =  \frac{1}{9},
        \end{align*}
        where we again use the independence of elements
    \item $k \neq l$ and $i \neq j$ \textbf{($m(m-1)n(n-1)$ such terms)}:
        \begin{align*}
            \mathbb{E}[a_{ki}a_{kj}a_{li}a_{lj}] = \mathbb{E}[a_{ki}]\mathbb{E}[a_{kj}]\mathbb{E}[a_{li}]\mathbb{E}[a_{lj}] =  0,
        \end{align*}
        where we again use the independence of elements.
    \end{enumerate}
    By summing over these elements and rearranging, we have that
    \begin{align*}
        \mathbb{E}[\|\mA^{\top}\mA\|_{F}^{2}] &= \frac{mn}{5} + \frac{m(m-1)n}{9} + \frac{mn(n-1)}{9} \\
        &= mn \left ( \frac{1}{5} + \frac{m+n-2}{9}  \right )
    \end{align*}
\end{proof}

Noting the rescaling of the uniform distribution, we may use this result to prove \autoref{lem:W1_rescaling}.

\begin{proof}[Proof of \autoref{thm:1_over_l_unbalanced}]
We may apply \autoref{lem:W1_recaling_helper} with $m = d_{l+1},\ n = d_{l}$ and divide by $d_l^{2}$ to account for the fact that the entries are drawn from $U(-1/\sqrt{d_l}, 1/\sqrt{d_l})$ rather than $U(-1, 1)$.
\end{proof}

\end{document}

%% file: math_commands.tex
\usepackage{amsmath,amsfonts,bm}

\def\eqref#1{equation~\ref{#1}}

\def\1{\bm{1}}

\def\vtheta{{\bm{\theta}}}
\def\va{{\bm{a}}}
\def\vb{{\bm{b}}}

\def\vv{{\bm{v}}}

\def\vx{{\bm{x}}}

\def\vz{{\bm{z}}}

\def\mA{{\bm{A}}}

\def\mC{{\bm{C}}}
\def\mD{{\bm{D}}}
\def\mE{{\bm{E}}}

\def\mJ{{\bm{J}}}

\def\mM{{\bm{M}}}
\def\mN{{\bm{N}}}

\def\mT{{\bm{T}}}
\def\mU{{\bm{U}}}
\def\mV{{\bm{V}}}
\def\mW{{\bm{W}}}
\def\mX{{\bm{X}}}
\def\mY{{\bm{Y}}}

\def\mSigma{{\bm{\Sigma}}}

\DeclareMathAlphabet{\mathsfit}{\encodingdefault}{\sfdefault}{m}{sl}
\SetMathAlphabet{\mathsfit}{bold}{\encodingdefault}{\sfdefault}{bx}{n}

\newcommand{\E}{\mathbb{E}}

\newcommand{\R}{\mathbb{R}}

\DeclareMathOperator{\Tr}{Tr}

\newcommand{\WonetWone}{{\mW_{1}^{\top}\mW_{1}}}
\newcommand{\WtW}{{\mW^{\top}\mW}}

\newcommand{\Af}{{\mA_{f}}}
\newcommand{\Ef}{{\mE_{f}}}
\newcommand{\AfStar}{{\mA_{f^{*}}}}

\newcommand{\AfTilde}{{\mA_{\tilde{f}}}}